\documentclass[conference,dvipsnames]{IEEEtran}
\usepackage{graphicx,url}
\graphicspath{ {./images/} }
\usepackage{amssymb,amsmath,mathtools}
\usepackage{tikz,pgfplots}
\usepackage{pgfplots}
\usepackage{pgfplotstable}
\usetikzlibrary{pgfplots.groupplots}
\usepgfplotslibrary{colorbrewer}
\pgfplotsset{compat = 1.15, cycle list/Set1-8}
\usetikzlibrary{pgfplots.statistics, pgfplots.colorbrewer}
\usetikzlibrary{pgfplots.groupplots}
\usetikzlibrary{shapes.geometric,backgrounds,patterns, trees}
\usetikzlibrary{3d,decorations.text,shapes.arrows,positioning,fit,backgrounds}
\usetikzlibrary{positioning, decorations.pathmorphing, shapes}
\usetikzlibrary{decorations.pathreplacing}
\usetikzlibrary{shapes.geometric,backgrounds,patterns, trees}
\usetikzlibrary{spy}
\usetikzlibrary{arrows.meta,
                bending,
                intersections,
                quotes,
                shapes.geometric}
              \usetikzlibrary{automata, positioning}
              \usepgfplotslibrary{fillbetween}
\usetikzlibrary{shapes,arrows}
\usetikzlibrary{arrows.meta}
\usetikzlibrary{positioning}
\tikzset{set/.style={draw,circle,inner sep=0pt,align=center}}
\usetikzlibrary{automata, positioning}
  \usetikzlibrary{shapes,shadows}
  \tikzstyle{abstractbox} = [draw=black, fill=white, rectangle,
  inner sep=10pt, style=rounded corners, drop shadow={fill=black,
  opacity=1}]
\tikzstyle{abstracttitle} =[fill=white]
\usetikzlibrary{calc,positioning,shapes.geometric}
\usetikzlibrary{arrows.meta,arrows}
\usetikzlibrary{matrix}

\colorlet{myRed}{red!20}
\tikzset{
  rows/.style 2 args={/utils/temp/.style={row ##1/.append style={nodes={#2}}},
    /utils/temp/.list={#1}},
  columns/.style 2 args={/utils/temp/.style={column ##1/.append style={nodes={#2}}},
    /utils/temp/.list={#1}}}
\usetikzlibrary{backgrounds,calc,shadings,shapes.arrows,shapes.symbols,shadows}
\definecolor{switch}{HTML}{006996}
\pgfkeys{/pgf/.cd,
  parallelepiped offset x/.initial=2mm,
  parallelepiped offset y/.initial=2mm
}
\pgfdeclareshape{parallelepiped}
{
  \inheritsavedanchors[from=rectangle] 
  \inheritanchorborder[from=rectangle]
  \inheritanchor[from=rectangle]{north}
  \inheritanchor[from=rectangle]{north west}
  \inheritanchor[from=rectangle]{north east}
  \inheritanchor[from=rectangle]{center}
  \inheritanchor[from=rectangle]{west}
  \inheritanchor[from=rectangle]{east}
  \inheritanchor[from=rectangle]{mid}
  \inheritanchor[from=rectangle]{mid west}
  \inheritanchor[from=rectangle]{mid east}
  \inheritanchor[from=rectangle]{base}
  \inheritanchor[from=rectangle]{base west}
  \inheritanchor[from=rectangle]{base east}
  \inheritanchor[from=rectangle]{south}
  \inheritanchor[from=rectangle]{south west}
  \inheritanchor[from=rectangle]{south east}
  \backgroundpath{
    \southwest \pgf@xa=\pgf@x \pgf@ya=\pgf@y
    \northeast \pgf@xb=\pgf@x \pgf@yb=\pgf@y
    \pgfmathsetlength\pgfutil@tempdima{\pgfkeysvalueof{/pgf/parallelepiped
      offset x}}
    \pgfmathsetlength\pgfutil@tempdimb{\pgfkeysvalueof{/pgf/parallelepiped
      offset y}}
    \def\ppd@offset{\pgfpoint{\pgfutil@tempdima}{\pgfutil@tempdimb}}
    \pgfpathmoveto{\pgfqpoint{\pgf@xa}{\pgf@ya}}
    \pgfpathlineto{\pgfqpoint{\pgf@xb}{\pgf@ya}}
    \pgfpathlineto{\pgfqpoint{\pgf@xb}{\pgf@yb}}
    \pgfpathlineto{\pgfqpoint{\pgf@xa}{\pgf@yb}}
    \pgfpathclose
    \pgfpathmoveto{\pgfqpoint{\pgf@xb}{\pgf@ya}}
    \pgfpathlineto{\pgfpointadd{\pgfpoint{\pgf@xb}{\pgf@ya}}{\ppd@offset}}
    \pgfpathlineto{\pgfpointadd{\pgfpoint{\pgf@xb}{\pgf@yb}}{\ppd@offset}}
    \pgfpathlineto{\pgfpointadd{\pgfpoint{\pgf@xa}{\pgf@yb}}{\ppd@offset}}
    \pgfpathlineto{\pgfqpoint{\pgf@xa}{\pgf@yb}}
    \pgfpathmoveto{\pgfqpoint{\pgf@xb}{\pgf@yb}}
    \pgfpathlineto{\pgfpointadd{\pgfpoint{\pgf@xb}{\pgf@yb}}{\ppd@offset}}
  }
}

\makeatletter
\tikzset{anchor/.append code=\let\tikz@auto@anchor\relax,
  add font/.code=%
    \expandafter\def\expandafter\tikz@textfont\expandafter{\tikz@textfont#1},
  left delimiter/.style 2 args={append after command={\tikz@delimiter{south east}
    {south west}{every delimiter,every left delimiter,#2}{south}{north}{#1}{.}{\pgf@y}}}}
\tikzstyle{sms} = [rectangle callout, draw,very thick, rounded corners, minimum height=20pt]
\makeatletter
\tikzset{anchor/.append code=\let\tikz@auto@anchor\relax,
  add font/.code=%
    \expandafter\def\expandafter\tikz@textfont\expandafter{\tikz@textfont#1},
  left delimiter/.style 2 args={append after command={\tikz@delimiter{south east}
    {south west}{every delimiter,every left delimiter,#2}{south}{north}{#1}{.}{\pgf@y}}}}
\tikzstyle{sms} = [rectangle callout, draw,very thick, rounded corners, minimum height=20pt]
\usetikzlibrary{positioning,calc}
\tikzstyle{block} = [rectangle, draw,
text width=10.5em, text centered, rounded corners, minimum height=4em]
\tikzstyle{line} = [draw, -latex]
\tikzset{l3 switch/.style={
    parallelepiped,fill=switch, draw=white,
    minimum width=0.75cm,
    minimum height=0.75cm,
    parallelepiped offset x=1.75mm,
    parallelepiped offset y=1.25mm,
    path picture={
      \node[fill=white,
        circle,
        minimum size=6pt,
        inner sep=0pt,
        append after command={
          \pgfextra{
            \foreach \angle in {0,45,...,360}
            \draw[-latex,fill=white] (\tikzlastnode.\angle)--++(\angle:2.25mm);
          }
        }
      ]
       at ([xshift=-0.75mm,yshift=-0.5mm]path picture bounding box.center){};
    }
  },
  ports/.style={
    line width=0.3pt,
    top color=gray!20,
    bottom color=gray!80
  },
  rack switch/.style={
    parallelepiped,fill=white, draw,
    minimum width=1.25cm,
    minimum height=0.25cm,
    parallelepiped offset x=2mm,
    parallelepiped offset y=1.25mm,
    xscale=-1,
    path picture={
      \draw[top color=gray!5,bottom color=gray!40]
      (path picture bounding box.south west) rectangle
      (path picture bounding box.north east);
      \coordinate (A-west) at ([xshift=-0.2cm]path picture bounding box.west);
      \coordinate (A-center) at ($(path picture bounding box.center)!0!(path
        picture bounding box.south)$);
      \foreach \x in {0.275,0.525,0.775}{
        \draw[ports]([yshift=-0.05cm]$(A-west)!\x!(A-center)$)
          rectangle +(0.1,0.05);
        \draw[ports]([yshift=-0.125cm]$(A-west)!\x!(A-center)$)
          rectangle +(0.1,0.05);
       }
      \coordinate (A-east) at (path picture bounding box.east);
      \foreach \x in {0.085,0.21,0.335,0.455,0.635,0.755,0.875,1}{
        \draw[ports]([yshift=-0.1125cm]$(A-east)!\x!(A-center)$)
          rectangle +(0.05,0.1);
      }
    }
  },
  server/.style={
    parallelepiped,
    fill=white, draw,
    minimum width=0.35cm,
    minimum height=0.75cm,
    parallelepiped offset x=3mm,
    parallelepiped offset y=2mm,
    xscale=-1,
    path picture={
      \draw[top color=gray!5,bottom color=gray!40]
      (path picture bounding box.south west) rectangle
      (path picture bounding box.north east);
      \coordinate (A-center) at ($(path picture bounding box.center)!0!(path
        picture bounding box.south)$);
      \coordinate (A-west) at ([xshift=-0.575cm]path picture bounding box.west);
      \draw[ports]([yshift=0.1cm]$(A-west)!0!(A-center)$)
        rectangle +(0.2,0.065);
      \draw[ports]([yshift=0.01cm]$(A-west)!0.085!(A-center)$)
        rectangle +(0.15,0.05);
      \fill[black]([yshift=-0.35cm]$(A-west)!-0.1!(A-center)$)
        rectangle +(0.235,0.0175);
      \fill[black]([yshift=-0.385cm]$(A-west)!-0.1!(A-center)$)
        rectangle +(0.235,0.0175);
      \fill[black]([yshift=-0.42cm]$(A-west)!-0.1!(A-center)$)
        rectangle +(0.235,0.0175);
    }
  },
}

\usetikzlibrary{calc, shadings, shadows, shapes.arrows}
\tikzset{cross/.style={cross out, draw=black, minimum size=2*(#1-\pgflinewidth), inner sep=0pt, outer sep=0pt},
cross/.default={1pt}}
\tikzset{%
  interface/.style={draw, rectangle, rounded corners, font=\LARGE\sffamily},
  ethernet/.style={interface, fill=yellow!50},
  serial/.style={interface, fill=green!70},
  speed/.style={sloped, anchor=south, font=\large\sffamily},
  route/.style={draw, shape=single arrow, single arrow head extend=4mm,
    minimum height=1.7cm, minimum width=3mm, white, fill=switch!20,
    drop shadow={opacity=.8, fill=switch}, font=\tiny}
}

\usepackage{float}
\definecolor{bluetwo}{RGB}{189, 213, 234}
\definecolor{bluethree}{RGB}{165, 193, 224}
\definecolor{bluefour}{RGB}{141, 169, 200}

\usepackage{amsthm}
\theoremstyle{plain}
\newtheorem{proposition}{Proposition}
\newtheorem{definition}{Definition}
\newtheorem{corollary}{Corollary}

\newtheorem{remark}{Remark}
\newtheorem{assumption}{Assumption}
\usepackage{booktabs}
\usepackage{colortbl}
\definecolor{lightgray}{gray}{0.9}
\allowdisplaybreaks

\definecolor{lightgray}{gray}{0.9}
\definecolor{lightgray}{gray}{0.9}
\definecolor{lightgreen}{rgb}{0.88, 1, 0.88}
\definecolor{lightred}{rgb}{1, 0.88, 0.88}
\definecolor{lightblue}{rgb}{0.88, 0.94, 1}
\definecolor{lightorange}{rgb}{1, 0.94, 0.88}

\usepackage{bbm}
\usepackage{bm}

\DeclareMathOperator*{\argmin}{arg\,min}

\definecolor{gray2}{HTML}{ededed}
\definecolor{gray3}{HTML}{F5F5F5}
\definecolor{RoyalAzure}{rgb}{0.0, 0.22, 0.66}
\definecolor{lightgray}{gray}{0.9}
\definecolor{lightgray}{gray}{0.9}
\definecolor{lightgreen}{rgb}{0.88, 1, 0.88}
\definecolor{lightred}{rgb}{1, 0.88, 0.88}
\definecolor{lightblue}{rgb}{0.88, 0.94, 1}
\definecolor{lightorange}{rgb}{1, 0.94, 0.88}

\makeatletter
\tikzset{
    database/.style={
        path picture={
            \draw (0, 1.5*\database@segmentheight) circle [x radius=\database@radius,y radius=\database@aspectratio*\database@radius];
            \draw (-\database@radius, 0.5*\database@segmentheight) arc [start angle=180,end angle=360,x radius=\database@radius, y radius=\database@aspectratio*\database@radius];
            \draw (-\database@radius,-0.5*\database@segmentheight) arc [start angle=180,end angle=360,x radius=\database@radius, y radius=\database@aspectratio*\database@radius];
            \draw (-\database@radius,1.5*\database@segmentheight) -- ++(0,-3*\database@segmentheight) arc [start angle=180,end angle=360,x radius=\database@radius, y radius=\database@aspectratio*\database@radius] -- ++(0,3*\database@segmentheight);
        },
        minimum width=2*\database@radius + \pgflinewidth,
        minimum height=3*\database@segmentheight + 2*\database@aspectratio*\database@radius + \pgflinewidth,
    },
    database segment height/.store in=\database@segmentheight,
    database radius/.store in=\database@radius,
    database aspect ratio/.store in=\database@aspectratio,
    database segment height=0.1cm,
    database radius=0.25cm,
    database aspect ratio=0.35,
  }
\makeatother

\usetikzlibrary{backgrounds}
\usetikzlibrary{patterns}
\pgfmathdeclarefunction{gauss}{3}{%
  \pgfmathparse{1/(#3*sqrt(2*pi))*exp(-((#1-#2)^2)/(2*#3^2))}%
}

\usetikzlibrary{spy}
\usetikzlibrary{arrows.meta,
                bending,
                intersections,
                quotes,
                shapes.geometric}
              \usetikzlibrary{automata, positioning}
              \usepgfplotslibrary{fillbetween}

\tikzset{%
  interface/.style={draw, rectangle, rounded corners, font=\LARGE\sffamily},
  ethernet/.style={interface, fill=yellow!50},
  serial/.style={interface, fill=green!70},
  speed/.style={sloped, anchor=south, font=\large\sffamily},
  route/.style={draw, shape=single arrow, single arrow head extend=4mm,
    minimum height=1.7cm, minimum width=3mm, white, fill=switch!20,
    drop shadow={opacity=.8, fill=switch}, font=\tiny}
}
\definecolor{switch}{HTML}{006996}

\tikzset{l3 switch/.style={
    parallelepiped,fill=switch, draw=white,
    minimum width=0.75cm,
    minimum height=0.75cm,
    parallelepiped offset x=1.75mm,
    parallelepiped offset y=1.25mm,
    path picture={
      \node[fill=white,
      circle,
      minimum size=6pt,
      inner sep=0pt,
      append after command={
        \pgfextra{
          \foreach \angle in {0,45,...,360}
          \draw[-latex,fill=white] (\tikzlastnode.\angle)--++(\angle:2.25mm);
        }
      }
      ]
      at ([xshift=-0.75mm,yshift=-0.5mm]path picture bounding box.center){};
    }
  },
  ports/.style={
    line width=0.3pt,
    top color=gray!20,
    bottom color=gray!80
  },
  rack switch/.style={
    parallelepiped,fill=white, draw,
    minimum width=1.25cm,
    minimum height=0.25cm,
    parallelepiped offset x=2mm,
    parallelepiped offset y=1.25mm,
    xscale=-1,
    path picture={
      \draw[top color=gray!5,bottom color=gray!40]
      (path picture bounding box.south west) rectangle
      (path picture bounding box.north east);
      \coordinate (A-west) at ([xshift=-0.2cm]path picture bounding box.west);
      \coordinate (A-center) at ($(path picture bounding box.center)!0!(path
      picture bounding box.south)$);
      \foreach \x in {0.275,0.525,0.775}{
        \draw[ports]([yshift=-0.05cm]$(A-west)!\x!(A-center)$)
        rectangle +(0.1,0.05);
        \draw[ports]([yshift=-0.125cm]$(A-west)!\x!(A-center)$)
        rectangle +(0.1,0.05);
      }
      \coordinate (A-east) at (path picture bounding box.east);
      \foreach \x in {0.085,0.21,0.335,0.455,0.635,0.755,0.875,1}{
        \draw[ports]([yshift=-0.1125cm]$(A-east)!\x!(A-center)$)
        rectangle +(0.05,0.1);
      }
    }
  },
  server/.style={
    parallelepiped,
    fill=white, draw,
    minimum width=0.35cm,
    minimum height=0.75cm,
    parallelepiped offset x=3mm,
    parallelepiped offset y=2mm,
    xscale=-1,
    path picture={
      \draw[top color=gray!5,bottom color=gray!40]
      (path picture bounding box.south west) rectangle
      (path picture bounding box.north east);
      \coordinate (A-center) at ($(path picture bounding box.center)!0!(path
      picture bounding box.south)$);
      \coordinate (A-west) at ([xshift=-0.575cm]path picture bounding box.west);
      \draw[ports]([yshift=0.1cm]$(A-west)!0!(A-center)$)
      rectangle +(0.2,0.065);
      \draw[ports]([yshift=0.01cm]$(A-west)!0.085!(A-center)$)
      rectangle +(0.15,0.05);
      \fill[black]([yshift=-0.35cm]$(A-west)!-0.1!(A-center)$)
      rectangle +(0.235,0.0175);
      \fill[black]([yshift=-0.385cm]$(A-west)!-0.1!(A-center)$)
      rectangle +(0.235,0.0175);
      \fill[black]([yshift=-0.42cm]$(A-west)!-0.1!(A-center)$)
      rectangle +(0.235,0.0175);
    }
  },
}

\makeatletter
\pgfdeclareradialshading[tikz@ball]{cloud}{\pgfpoint{-0.275cm}{0.4cm}}{%
  color(0cm)=(tikz@ball!75!white);
  color(0.1cm)=(tikz@ball!85!white);
  color(0.2cm)=(tikz@ball!95!white);
  color(0.7cm)=(tikz@ball!89!black);
  color(1cm)=(tikz@ball!75!black)
}
\tikzoption{cloud color}{\pgfutil@colorlet{tikz@ball}{#1}%
  \def\tikz@shading{cloud}\tikz@addmode{\tikz@mode@shadetrue}}
\makeatother

\tikzset{my cloud/.style={
     cloud, draw, aspect=2,
     cloud color={gray!5!white}
  }
}

\tikzset{
  mybackground18/.style={execute at end picture={
      \begin{scope}[on background layer]
        \draw[black, fill=gray3, rounded corners=3.5ex] (current bounding box.south west)
        rectangle (current bounding box.north east);
        \node[draw,fill=white,ellipse,anchor=west,inner sep=1pt,minimum width=4ex] at (current bounding box.north
        west){#1};
      \end{scope}
    }}
}

\usepackage{listofitems} 
\usetikzlibrary{arrows.meta} 
\usepackage[outline]{contour} 
\contourlength{1.4pt}

\colorlet{myred}{red!80!black}
\colorlet{myblue}{blue!80!black}
\colorlet{mygreen}{green!60!black}
\colorlet{myorange}{orange!70!red!60!black}
\colorlet{mydarkred}{red!30!black}
\colorlet{mydarkblue}{blue!40!black}
\colorlet{mydarkgreen}{green!30!black}

\tikzset{
  >=latex, 
  node/.style={thick,circle,draw=myblue,minimum size=22,inner sep=0.5,outer sep=0.6},
  node in/.style={node,black!20!black,draw=mygreen!30!black,fill=black!20},
  node hidden/.style={node,black!20!black,draw=myblue!30!black,fill=black!20},
  node convol/.style={node,black!20!black,draw=myorange!30!black,fill=black!20},
  node out/.style={node,red!20!black,draw=myred!30!black,fill=black!20},
  connect/.style={thick,Blue!100}, 
  connect arrow/.style={-{Latex[length=4,width=3.5]},thick,mydarkblue,shorten <=0.5,shorten >=1},
  node 1/.style={node in}, 
  node 2/.style={node hidden},
  node 3/.style={node out}
}

\usepackage{makecell}
\usepackage[most]{tcolorbox}

\usepackage[titlenumbered,ruled,linesnumbered]{algorithm2e}
\SetAlCapNameFnt{\footnotesize}
\SetAlCapFnt{\footnotesize}

\usepackage{pifont}
\newcommand{\cmark}{\textcolor{OliveGreen}{\ding{51}}} 
\newcommand{\xmark}{\textcolor{Red}{\ding{55}}}  
\newcommand{\qmark}{\textcolor{Blue}{\textbf{?}}}

\begin{document}

\title{Hallucination-Resistant Security Planning\\
with a Large Language Model
}

\author{\IEEEauthorblockN{1\textsuperscript{st} Kim Hammar}
\IEEEauthorblockA{
\textit{The University of Melbourne}\\
Melbourne, Australia \\
kim.hammar@unimelb.edu.au}
\and
\IEEEauthorblockN{2\textsuperscript{nd} Tansu Alpcan}
\IEEEauthorblockA{
\textit{The University of Melbourne}\\
Melbourne, Australia \\
tansu.alpcan@unimelb.edu.au}
\and
\IEEEauthorblockN{3\textsuperscript{rd} Emil C. Lupu}
\IEEEauthorblockA{
\textit{Imperial College London}\\
London, United Kingdom \\
e.c.lupu@imperial.ac.uk}
}

\maketitle
\begin{abstract}
Large language models (LLMs) are promising tools for supporting security management tasks, such as incident response planning. However, their unreliability and tendency to hallucinate remain significant challenges. In this paper, we address these challenges by introducing a principled framework for using an LLM as decision support in security management. Our framework integrates the LLM in an iterative loop where it generates candidate actions that are checked for consistency with system constraints and lookahead predictions. When consistency is low, we abstain from the generated actions and instead collect external feedback, e.g., by evaluating actions in a digital twin. This feedback is then used to refine the candidate actions through in-context learning (ICL). We prove that this design allows to control the hallucination risk by tuning the consistency threshold. Moreover, we establish a bound on the regret of ICL under certain assumptions. To evaluate our framework, we apply it to an incident response use case where the goal is to generate a response and recovery plan based on system logs. Experiments on four public datasets show that our framework reduces recovery times by up to $30$\% compared to frontier LLMs.
\end{abstract}


\IEEEpeerreviewmaketitle

\section{Introduction}
\IEEEPARstart{M}{anaging} the security of networked systems alongside their service requirements and physical infrastructures is an arduous technical challenge that continues to grow. Examples of management tasks include incident response, risk analysis, policy design, and threat hunting. Today, many of these tasks remain manual processes carried out by security experts. While this approach can be effective, it is labor-intensive and requires significant skills. For example, a recent study reports a global shortage of more than $4$ million security experts \cite{ISC2_2024_Workforce_Study}. 

\begin{figure}
  \centering 
  \scalebox{0.79}{
    \includegraphics{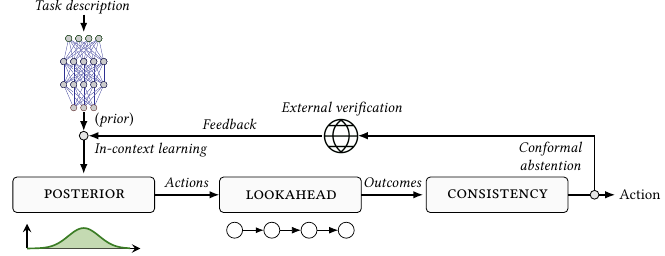}
  }
\caption{Our framework for hallucination-resistant security planning with a large language model (LLM). We integrate the LLM in an iterative verification and refinement loop, in which the LLM is used to generate candidate actions that are checked for consistency with lookahead predictions. When consistency is low, we abstain from the generated actions and collect external feedback, which allows to refine the candidate actions through in-context learning (ICL).}
  \label{fig:method}
\end{figure}

To address this shortage, an emerging direction of research is to leverage the security knowledge encoded in large language models (LLMs) as decision support in security management. Notably, IBM recently launched an LLM-based incident response service \cite{hussey2025instana} and Google has developed an LLM-driven security operations platform \cite{google_ai_sec}. Other applications of LLMs that are explored in the research literature include penetration testing \cite{pentest_gpt,rodriguez2025frameworkevaluatingemergingcyberattack}, security assistants \cite{DBLP:conf/ndss/DengLCBWLW025}, scanning \cite{DBLP:conf/ndss/StafeevRSKP25,299549}, threat hunting \cite{google_llm_recovery}, verification \cite{DBLP:conf/ndss/0012XW00S025}, piracy \cite{DBLP:conf/ndss/GohilDNSR25}, detection \cite{DBLP:conf/ndss/YangL0L25}, fuzzing \cite{10.1145/3597503.3639121,299896}, API design \cite{DBLP:conf/ndss/LiuY0L25}, network operations \cite{net_llm,10575237}, intent-based networking \cite{11073595,11073741}, deployment planning \cite{11073607}, threat intelligence \cite{ARAZZI2025100765}, and decompilation \cite{DBLP:conf/ndss/HuL024}.

Most of the methods proposed so far are based on prompt engineering of frontier LLMs, such as \textsc{gpt-5} \cite{openai2024gpt4technicalreport}. While this approach can be effective, it relies on extensive tuning of prompts and does not provide theoretical guarantees. Moreover, prompt engineering offers no way of mitigating \textit{hallucinations}, i.e., cases where the LLM generates outputs that appear plausible but are incorrect. This limitation is especially critical in security management, where incorrect outputs can lead to costly service disruptions and misconfigurations.

In this paper, we address these limitations by presenting a principled framework for using an LLM as decision support in security management; see Fig. \ref{fig:method}. Our framework embeds the LLM in an iterative loop where it generates candidate actions that are checked for consistency with system constraints and lookahead predictions. When consistency is low, we abstain from the generated actions and instead collect external feedback, e.g., by evaluating actions in a digital twin. This feedback is then incorporated into the context of the LLM and used to generate a new set of candidate actions. In other words, the candidate actions are refined through in-context learning (ICL). We prove that this design allows our framework to achieve a desirable hallucination risk by tuning the consistency threshold. Moreover, we establish an upper bound on the regret of ICL under certain assumptions. 

We evaluate our framework experimentally by applying it to an incident response use case where the goal is to generate a response and recovery plan based on system logs and security alerts. Experiments on four public datasets show that our framework reduces hallucination and shortens recovery times by up to $30$\% compared to frontier LLMs (e.g., \textsc{openai o3}). We also present an ablation study assessing the contribution of the individual steps of our framework. This study shows that each step has a positive impact on its performance.

In summary, the main contributions of this paper are:
\begin{itemize}
\item We develop a novel framework for using an LLM as decision support in security management. Our framework combines lookahead, consistency, and ICL. We provide a video demonstration of our framework at \cite{llm_source_kim}.
\item We derive a bound on the Bayesian regret of ICL and prove that our framework can achieve a desirable hallucination risk by tuning the consistency threshold.
\item We evaluate our framework experimentally on an incident response use case. Our evaluation is based on incidents from four different datasets published in the literature. The results show that our framework reduces hallucination and shortens recovery times by up to $30$\% compared to frontier LLMs, such as \textsc{gemini 2.5 pro} and \textsc{openai o3}. Our testbed is open-source and available at \cite{llm_source_kim}.
\end{itemize}

\section{Related Work}\label{sec:related_work}
A growing body of work explores the use of LLMs in security management; see e.g., \cite{pentest_gpt,rodriguez2025frameworkevaluatingemergingcyberattack,DBLP:conf/ndss/DengLCBWLW025,DBLP:conf/ndss/StafeevRSKP25,299549,hammar2025incidentresponseplanningusing}. For example, Deng et al. present a framework for penetration testing using LLMs \cite{pentest_gpt}. Similarly, Stafeev et al. introduce a framework for security scanning with LLMs \cite{DBLP:conf/ndss/StafeevRSKP25}. These studies demonstrate promising applications but are primarily concerned with systems aspects such as architecture design and information retrieval. To our knowledge, no prior work \textit{in security management} has presented an LLM-based framework that provides theoretical guarantees. In particular, current approaches are mainly based on prompt engineering, which is unreliable and does not mitigate hallucinations; see Table \ref{tab:related_work}.

By contrast, reliable decision-making with LLMs has emerged as a major research topic in artificial intelligence; see e.g., \cite{yadkori2024mitigatingllmhallucinationsconformal,tayebati2025learningconformalabstentionpolicies,DBLP:conf/iclr/0002WSLCNCZ23,chen2023universalselfconsistencylargelanguage,elazar-etal-2021-measuring}. Notable contributions within this literature are the concepts of self-consistency \cite{DBLP:conf/iclr/0002WSLCNCZ23}, conformal abstention \cite{yadkori2024mitigatingllmhallucinationsconformal,tayebati2025learningconformalabstentionpolicies}, and ICL \cite{moeini2025surveyincontextreinforcementlearning}, all of which can be used to improve the reliability of LLM-based decision-making. While our framework builds on these concepts, we introduce a new way to incorporate feedback in the decision-making process. Moreover, we are the first to apply these concepts in the context of security management.

 \begin{table}[H]
  \centering
  \scalebox{0.82}{
    \begin{tabular}{llll} \toprule
\rowcolor{lightgray}
      {\textit{Approach}} & {\textit{General framework}} & {\textit{Hallucination-resistant}} & {\textit{Theory}} \\ \midrule
    \rowcolor{lightgreen}
      Our framework  & \cmark & \cmark & \cmark\\
      \cite{hussey2025instana,google_ai_sec,google_llm_recovery}  & \qmark & \qmark & \xmark\\
      \cite{pentest_gpt,10.1145/3597503.3639121}  & \cmark & \xmark & \xmark\\
      \cite{net_llm,11073595,castro2025largelanguagemodelsautonomous,yan2024dependingshouldmentoringllm}  & \cmark & \xmark & \xmark\\
      \cite{rodriguez2025frameworkevaluatingemergingcyberattack,DBLP:conf/ndss/DengLCBWLW025,DBLP:conf/ndss/StafeevRSKP25,DBLP:conf/ndss/HuL024,10.1145/3719027.3744855}  & \xmark & \xmark & \xmark\\
      \cite{299549,DBLP:conf/ndss/0012XW00S025,DBLP:conf/ndss/GohilDNSR25,DBLP:conf/ndss/YangL0L25,299740} & \xmark & \xmark & \xmark\\
      \cite{299896,DBLP:conf/ndss/LiuY0L25,10575237,10.1145/3719027.3765219} & \xmark & \xmark & \xmark\\
      \cite{11073741,11073607,10991969,10.1145/3719027.3744872} & \xmark & \xmark & \xmark\\
    \bottomrule\\
  \end{tabular}}
\caption{Comparison between this paper (green row) and related work that uses LLMs to support security management tasks.}\label{tab:related_work}
\end{table}

\section{Problem Statement}
We consider the problem of completing an open-ended security management task that requires executing a sequence of actions $\mathbf{a}_0,\mathbf{a}_1,\hdots,\mathbf{a}_{T-1}$. We place no restriction on the structure of an action apart from assuming that it can be expressed in textual form. For example, an action can be a system command, a configuration change, or an instruction to a security operator. The goal when selecting these actions is to complete the task as quickly as possible, i.e., to minimize the \textit{task completion time} $T$. No mathematical model or simulator of the task is assumed, which means that it cannot be addressed with traditional planning or optimization methods.
\subsection*{Motivating Example: Incident Response}\label{sec:use_case}
An example of a security management task that fits the general description above is \textit{incident response}, which involves selecting a sequence of actions that restores a networked system to a secure and operational state after a cyberattack. When selecting these actions, a key challenge is that the information about the attack is often limited to partial indicators of compromise, such as logs and alerts. Another major challenge is that the actions have to be tailored to the context of the incident. Example actions include redirecting network flows, patching vulnerabilities, and updating access control policies.

Figure~\ref{fig:resilience} illustrates the phases of incident response. Following the attack are detection and response time intervals, which represent the time to detect the attack and form a response, respectively. These phases are followed by a \textit{recovery time} interval $T$, during which response actions are deployed. The objective is to keep this interval as short as possible to limit the cost of the incident. For example, in the event of a ransomware attack, a delay of a few minutes in containing the attack may allow the malware to encrypt systems or spread laterally \cite{wannacry_nhgs}.

\begin{figure}[H]
  \centering 
  \scalebox{0.88}{
    \includegraphics{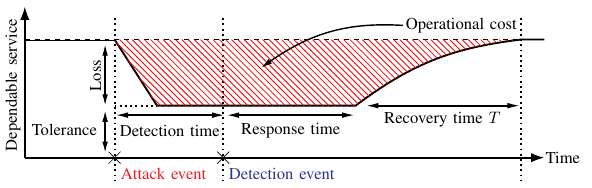}
  }
  \caption{Phases and performance metrics of the incident response use case.}
  \label{fig:resilience}
\end{figure}

\begin{figure*}
  \centering
  \scalebox{1.45}{
    \includegraphics{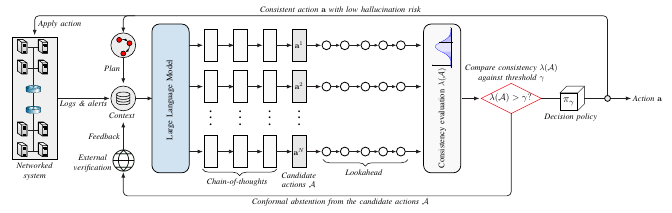}
  }
\caption{Our framework for hallucination-resistant security planning with a large language model (LLM).
 Our framework integrates the LLM inside an iterative loop of verification and refinement. In this loop, the LLM is prompted with details about a management task and generates candidate actions, which are then evaluated for consistency against lookahead predictions. If the consistency is low, the framework abstains from selecting an action and instead gathers external feedback, e.g., by testing the actions in a digital twin or asking a security expert. This feedback is subsequently leveraged to improve the candidate actions through in-context learning (ICL). The iterative verification and refinement procedure continues until an action that meets the consistency criterion is found.}
  \label{fig:deployment}
\end{figure*}
\section{Our Framework for Hallucination-Resistant \\Security Planning with an LLM}
In this section, we present our solution framework for addressing the problem stated above. In particular, we present a practical framework for using an LLM as decision support in security management while providing theoretical guarantees.

\subsection{Framework Overview}
Figure~\ref{fig:deployment} provides an architectural overview of our framework. As shown in the figure, our framework integrates the LLM within an iterative loop where it is used to generate candidate actions that are checked for consistency with system constraints and lookahead predictions. When consistency is low, we abstain from the generated actions and instead collect external feedback, e.g., by evaluating actions in a digital twin or asking a security expert. This feedback is then used to refine the candidate actions through in-context learning (ICL). In the following subsections, we describe each of these steps in detail, starting with the generation of candidate actions.
\subsection{Using the LLM to Generate Candidate Actions}\label{sec:candidates}
At each stage $t$ of the task, we prompt the LLM with information about the task and previously applied actions. We then instruct the LLM to generate $N$ \textit{candidate actions} $\mathcal{A}_t=\{\mathbf{a}_t^1,\hdots,\mathbf{a}_t^N\}$ to apply next. For example, an action can be generated via \textit{chain-of-thought prompting} \cite{10.5555/3600270.3602070}. In this case, the LLM is instructed to analyze the task via step-by-step reasoning before generating the next action. We provide an example of a chain-of-thought instruction in Appendix \ref{app:prompt}.

Given such an instruction, the LLM generates an action by predicting the most likely sequence of tokens that follow the prompt. While this approach is flexible and can adapt to the prompt, it offers no guarantees. Consequently, the LLM may generate actions that seem plausible yet are incorrect. This failure mode is commonly referred to as \textit{hallucination} \cite{kalai2025languagemodelshallucinate}.

In the context of security planning, we consider an action to be hallucinated if it does not contribute to the goal of completing the task, as expressed by the following definition.
\begin{definition}[Hallucinated action]\label{def:hallucination}
Given a sequence of actions $\mathbf{a}_0,\hdots,\mathbf{a}_{t-1},\mathbf{a}_t$ for completing a security management task. The action $\mathbf{a}_t$ is hallucinated if it does not reduce the expected time remaining to complete the task, i.e.,
\begin{align}
\mathbb{E}\{T_{t+1} \mid \mathbf{a}_0,\hdots,\mathbf{a}_t\} \geq \mathbb{E}\{T_t \mid \mathbf{a}_0,\hdots,\mathbf{a}_{t-1}\},
\end{align}
where $T_t$ is the time remaining to complete the task at stage $t$, assuming that future actions are selected optimally.
\end{definition}
To illustrate the preceding definition, consider the incident response use case described in \S\ref{sec:use_case}. Suppose that the LLM generates an action that aims to patch a vulnerability that does not exist. This action does not make any progress toward recovering from the incident and is therefore a hallucination in our definition. Similarly, the LLM could generate an action that involves shutting down a system component that is unrelated to the incident. This action is counterproductive to recovery as it creates unnecessary disruptions that need to be resolved, which increases the time to recover from the incident. Consequently, it is a hallucination according to Def.~\ref{def:hallucination}.

Given the possibility of such hallucinations, we require a reliable way to evaluate and refine the candidate actions to reduce the hallucination risk. In our framework, we address this problem by evaluating the \textit{consistency} of the candidate actions, as detailed in the following subsection.

\subsection{Evaluating the Consistency of Candidate Actions}\label{sec:assessment}
Detecting hallucinations is challenging because it requires knowledge of the effects of actions, which are typically unknown. To address this challenge, we instead assess the \textit{consistency} of actions generated by the LLM and use inconsistency as an indication of hallucination. From a logical perspective, the LLM is \textit{inconsistent} if it contradicts itself, i.e., if it produces inconsistent outputs for the same input.

Building on the logical foundation, we measure the LLM's consistency by checking for contradictions among the candidate actions in the set $\mathcal{A}_t$. Intuitively, if the candidate actions contradict each other, it indicates a higher risk of hallucination since the LLM is producing conflicting actions when queried multiple times with the same instruction. This intuition is supported by numerous empirical studies that found consistency to reduce hallucination; see e.g., \cite{weng-etal-2023-large,DBLP:conf/iclr/0002WSLCNCZ23,chen2023universalselfconsistencylargelanguage}.

To measure conflicts between the candidate actions, we compare their expected outcomes based on \textit{lookahead predictions}. Specifically, we use the LLM to predict the effect of each candidate action $\mathbf{a}^i_t$ in terms of the expected time remaining to complete the task after executing the action, which we denote by $T^i_{t+1}$. We then measure the consistency of the set $\mathcal{A}_t$ based on the dispersion of these predictions. Specifically, we quantify the consistency of the set $\mathcal{A}_t$ via the consistency function
\begin{align}
\lambda(\mathcal{A}_t) = \exp\left(\frac{-\beta}{N}\sum_{i=1}^N\left(T_{t+1}^i - \overline{T}_{t+1}\right)^2\right),\label{eq:self_consistency}
\end{align}
where $\beta > 0$ is a constant that determines how quickly the consistency decays as the disagreement among the predictions increases and $\overline{T}_{t+1}$ is the average value, i.e.,
$$\overline{T}_{t+1}=\frac{1}{N}\sum_{i=1}^NT^i_{t+1}.$$
The exponential in \eqref{eq:self_consistency} ensures that the consistency is bounded between $0$ (inconsistent) and $1$ (consistent). We illustrate this consistency function through two examples in Appendix~\ref{app:consistency}.
\begin{remark}
The consistency function $\lambda$ does not rely on task-specific consistency constraints. If such constraints are available, they can be incorporated in \eqref{eq:self_consistency}. For example, in the context of incident response [cf.~\S \ref{sec:use_case}], response actions may be required to preserve service availability for clients.
\end{remark}

In our framework, we use the consistency function $\lambda$ [cf.~\eqref{eq:self_consistency}] to abstain from actions with low consistency. If abstention does not occur, we instead select the action that leads to the shortest (predicted) time to complete the task, as expressed by the following policy:
\begin{align}
  \pi_{\gamma}(\mathcal{A}_t) &=
                                     \begin{dcases}
                                       \emptyset \text{ (abstain)}, & \text{if } \lambda(\mathcal{A}_t) \leq \gamma,\\
                                       \argmin_{\mathbf{a}^i_t \in \mathcal{A}_t}\{T_{t+1}^i\}, &  \text{if } \lambda(\mathcal{A}_t) > \gamma,
                                     \end{dcases}  \label{eq:threshold_rule}
\end{align}
where ties in the $\argmin$ are broken randomly and $\gamma \in [0,1]$ is a \textit{consistency threshold} for controlling the hallucination risk.
\subsection{In-Context Learning from Feedback}\label{sec:collecting_feedback}
When the policy $\pi_{\gamma}$ [cf.~\eqref{eq:threshold_rule}] abstains, it indicates that the candidate actions do not meet the required level of consistency. To improve the consistency, we collect external feedback by evaluating the action that would have been chosen if we did not abstain, e.g., by executing it in a digital twin and observing its effect; see Fig.~\ref{fig:digital_twin}. We then incorporate the feedback into the LLM’s context and generate a new set of candidate actions. In other words, the candidate actions are refined through in-context learning (ICL). This approach allows the LLM to learn about effective actions without changing its underlying parameters; see e.g., \cite{dong-etal-2024-survey} and \cite{10.5555/3495724.3495883} for details about ICL.

The iterative procedure of generating actions and collecting feedback continues until the actions meet the consistency criterion in \eqref{eq:threshold_rule}. Since the collection of feedback takes time and resources, the consistency threshold $\gamma$ controls a trade-off between the hallucination risk and the cost of feedback collection. We analyze this trade-off in the next section.

\begin{figure}
  \centering
  \scalebox{1.05}{
    \includegraphics{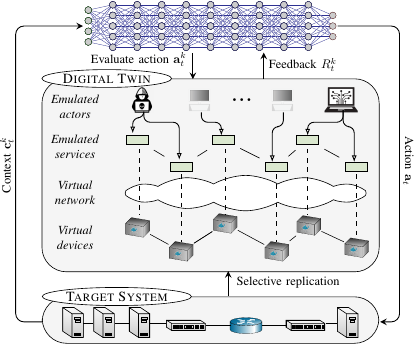}
 }
\caption{Architecture for collecting feedback by evaluating the effects of actions using a digital twin, i.e., a virtual replica of the target system \cite{10154288,hammar_stadler_tnsm}.}
\label{fig:digital_twin}
\end{figure}

\section{Theoretical Analysis of Our Framework}\label{sec:theory}
In this section, we present a theoretical analysis of our framework. Specifically, we present two main results: (\textit{i}) the consistency threshold $\gamma$ can be tuned to achieve a desirable upper bound on the hallucination probability; and (\textit{ii}) the regret of ICL is upper bounded under certain assumptions.
\subsection{Controlling the Hallucination Probability}
Suppose that we have a \textit{calibration dataset} $\{\mathcal{A}^i\}^n_{i=1}$ of sets of candidate actions such that $\pi_{0}(\mathcal{A}^{i})$ is a hallucinated action for each $\mathcal{A}^{i}$. We can then tune $\gamma$ to achieve a desirable upper bound on the hallucination probability, as stated below.
\begin{proposition}\label{prop:abstention_bound}
Assume the sets $\{\mathcal{A}^i\}^n_{i=1}$ are independent and identically distributed (i.i.d.). Let $\tilde{\mathcal{A}}$ be a test example from the same distribution and let $\kappa \in (0,1]$ be a desirable upper bound on the hallucination probability.

Define the threshold
\begin{align}
\tilde{\gamma} = \inf\left\{\gamma \text{ }\middle|\text{ } \frac{|\{i \mid \lambda(\mathcal{A}^i) \leq \gamma \}|}{n} \geq \frac{\lceil (n+1)(1-\kappa)\rceil}{n}\right\}, \label{eq:abs_thresh}
\end{align}
where $\lceil \cdot \rceil$ is the ceiling function. We have
\begin{align*}
P\left(\pi_{\tilde{\gamma}}(\tilde{\mathcal{A}})\neq \emptyset \right) \leq \kappa.
\end{align*}
\end{proposition}
\begin{proof}
We start by noting that if $1-\kappa > 1-\frac{1}{n+1}$, then
\begin{equation}\label{eq:split_1}
\begin{split} 
  \lceil (n+1)(1-\kappa) \rceil &> \left\lceil (n+1)\left(1-\frac{1}{n+1}\right)\right\rceil\\
                            &= \lceil n \rceil = n,
\end{split}                              
\end{equation}
which implies that $\tilde{\gamma}=\infty$; cf. \eqref{eq:abs_thresh}. In this case, the rule \eqref{eq:threshold_rule} will always abstain and thus the proposition holds.

Now consider the case when $(1-\kappa) \in [0, 1-\frac{1}{n+1}]$. Let $\lambda_1,\lambda_2,\hdots,\lambda_{n}$ be an ordered sequence of the consistency values in the set $\{\lambda(\mathcal{A}^i)\}^n_{i=1}$ where ties are broken uniformly at random. Moreover, let $\tilde{\lambda}=\lambda(\tilde{\mathcal{A}})$. Since the values $\lambda_1,\lambda_2,\hdots,\lambda_n,\tilde{\lambda}$ are i.i.d., we have
\begin{align*}
P(\tilde{\lambda} \leq \lambda_k) = \frac{k}{n+1}, && \text{for all }k=1,\hdots,n.
\end{align*}
Because $(1-\kappa) \in [0, 1-\frac{1}{n+1}]$, it follows from \eqref{eq:split_1} that
$$\lceil (n+1)(1-\kappa) \rceil \leq n.$$
This means that $\tilde{\gamma} = \lambda_{\lceil (n+1)(1-\kappa)\rceil}$. As a consequence, the event $\{\pi_{\tilde{\gamma}}(\tilde{\mathcal{A}})\neq \emptyset\}$ is equivalent to the event
$$\{\tilde{\lambda} > \lambda_{\lceil (n+1)(1-\kappa)\rceil}\}.$$
Therefore, we have
\begin{align*}
P\left(\pi_{\tilde{\gamma}}(\tilde{\mathcal{A}})\neq \emptyset\right) &= P\big(\tilde{\lambda} > \lambda_{\lceil (n+1)(1-\kappa)\rceil}\big)\\
                                                               &= 1-P\big(\tilde{\lambda} \leq \lambda_{\lceil (n+1)(1-\kappa)\rceil}\big)\\  
                                                               &= 1-\frac{\lceil (n+1)(1-\kappa) \rceil}{n+1}\\
                                                               &\leq 1-\frac{(n+1)(1-\kappa)}{n+1}\\
                                                                      &= 1-(1-\kappa)\\
                                                                      &=\kappa.
\end{align*}
\end{proof}
Proposition~\ref{prop:abstention_bound} implies that the threshold~$\gamma$ balances the trade-off between abstention frequency and hallucination risk. By tuning this threshold, our framework can be tailored to the reliability requirements of different management tasks. For example, a low hallucination risk may be required in tasks such as automated incident response, whereas a higher risk can be acceptable in advisory settings, such as using the LLM as decision support in threat-hunting or vulnerability analysis.
\begin{remark}
Proposition~\ref{prop:abstention_bound} does not rely on the implementation of the consistency function $\lambda$. Hence, our framework can be instantiated with different consistency functions (e.g., functions for specific tasks) while retaining its theoretical guarantees.
\end{remark}
\subsection{Convergence of In-Context Learning}
While the preceding analysis shows that our framework can control the hallucination probability by abstaining from inconsistent actions, it does not guarantee that our framework will eventually output a consistent action as more feedback is collected. In this subsection, we establish conditions under which such convergence is guaranteed in our framework.

From a theoretical perspective, ICL can be interpreted as performing approximate Bayesian learning; see e.g., \cite{10.5555/3692070.3692580,10.5555/3692070.3693331,xie2022explanationincontextlearningimplicit}. To formalize this perspective, let $\mathbf{c}_t^k$ be the available information during iteration $k$ of ICL at stage $t$ of the task. Moreover, let $p_{\bm{\theta}}$ be the LLM-distribution from which the action in \eqref{eq:threshold_rule} is sampled. Following the Bayesian perspective, we assume that the distribution $p_{\bm{\theta}}(\cdot \mid \mathbf{c}^k_t)$ is aligned with the posterior $P(\cdot \mid \mathbf{c}_t^k)$, as formally expressed below.
\begin{assumption}[Bayesian learning]\label{assumption:2}
For the purpose of theoretical analysis, we assume that the distribution of actions in \eqref{eq:threshold_rule} when not abstaining is aligned with the posterior distribution of the optimal action, i.e.,
\begin{align*}
p_{\bm{\theta}}(\mathbf{a}_t \mid \mathbf{c}_t^k) &= P(\mathbf{a}_t \mid \mathbf{c}_t^k),
\end{align*}
for all information vectors $\mathbf{c}_t^k$ and actions $\mathbf{a}_t$.
\end{assumption}
We place no restriction on the form of feedback in our framework. It can range from database lookups to evaluations in a digital twin or demonstrations by a security expert. However, for the purpose of theoretical analysis, we assume that the feedback has a specific form, as expressed below.
\begin{assumption}[Bandit feedback]\label{assumption_reward}
Let $\mathbf{a}^k_t$ be the action that would have been selected at iteration $k$ of ICL if abstention did not occur. For the purpose of theoretical analysis, we assume that the feedback on the quality of action $\mathbf{a}_t^k$ is a real-valued reward $R_t^k$ that is an unbiased sample of the action's impact on the expected time to complete the task, i.e.,
\begin{align*}
\mathbb{E}\{R_t^k\} &= \mathbb{E}\{T_t \mid \mathbf{a}_0,\hdots,\mathbf{a}_{t-1}\} - \mathbb{E}\{T_{t+1} \mid \mathbf{a}_0,\hdots,\mathbf{a}_{t-1},\mathbf{a}_t^k\},
\end{align*}
where $T_t$ is the time remaining to complete the task at stage $t$, assuming that future actions are selected optimally.
\end{assumption}
To quantify the efficiency of ICL in discovering an effective action for completing the task, we use the Bayesian regret metric. Specifically, let $\mathcal{A}$ be the (finite) set of feasible actions and let $R_t^{\star}$ be the reward of an optimal action at stage $t$. The Bayesian regret after $K$ iterations of ICL is then given by
\begin{align}\label{eq:bayes_regret}
\mathcal{R}_t^K &= \mathbb{E}\left\{\sum_{k=0}^{K-1}\left(R_t^{\star} - R_t^k\right)\right\}.
\end{align}    
Given this definition of regret, we have the following result.
\begin{proposition}\label{prop:regret_bound}
Under Assumptions \ref{assumption:2}-\ref{assumption_reward}, we have
\begin{align*}
\mathcal{R}_t^K \leq C\sqrt[]{|\mathcal{A}|K\ln K}, &&\text{for each stage }t,
\end{align*}
where $\mathcal{R}_t^K$ is the Bayesian regret [cf.~\eqref{eq:bayes_regret}], $C > 0$ is a universal constant, $\mathcal{A}$ is the (finite) set of feasible actions for the management task, and $K$ is the number of ICL iterations.
\end{proposition}  
\begin{table*}
  \centering
  \scalebox{0.9}{
    \begin{tabular}{llll} \toprule
\rowcolor{lightgray}
      {\textit{Dataset}} & {\textit{System}} & {\textit{Attacks}} & {\textit{Logs}}\\ \midrule
      CTU-Malware-2014 \cite{GARCIA2014100} & Windows XP SP2 servers & Various malwares and ransomwares, e.g., CryptoDefense \cite{8418627}. & Snort alerts \cite{snort} \\
      CIC-IDS-2017 \cite{icissp18} & Windows and Linux servers & Denial-of-service, web attacks, heartbleed, SQL injection, etc. & Snort alerts \cite{snort} \\
      AIT-IDS-V2-2022 \cite{ait_ids_1} & Linux and Windows servers/hosts & Multi-stage attack with reconnaissance, cracking, and escalation. & Wazuh alerts \cite{wazuh} \\
      CSLE-IDS-2024 \cite{dsn24_hammar_stadler} & Linux servers & SambaCry, Shellshock, exploit of CVE-2015-1427, etc. & Snort alerts \cite{snort}\\
    \bottomrule\\
  \end{tabular}}
  \caption{The datasets of cyberattacks and logs used for the experimental evaluation; see Appendix~\ref{app:experimental_setup} for preprocessing details.}\label{tab:dataset_types}
\end{table*}  
\begin{proof}
We frame the ICL procedure at each stage $t$ of the task as a multi-armed bandit problem, where the set $\mathcal{A}$ corresponds to the arms of the bandit. At each ICL iteration, policy $\pi_{\gamma}$ [cf.~\eqref{eq:threshold_rule}] selects an action $\mathbf{a}_t^k$ and receives a reward $R_t^k$, as defined in Assumption \ref{assumption_reward}. Assumption~\ref{assumption:2} implies that the action is sampled from the posterior distribution $P(\mathbf{a}_t \mid \mathbf{c}^k_t)$. As a consequence, the action is sampled in accordance with the Thompson sampling algorithm \cite{thompson}. Therefore, the statement of the proposition follows from standard regret bounds for Thompson sampling; see e.g., \cite[Thm 36.1]{bandit_book}.
\end{proof}
The following corollary is immediate.
\begin{corollary}
Under Assumptions \ref{assumption:2}-\ref{assumption_reward} and assuming that the optimal action at each stage $t$ is unique, then the ICL process at each stage $t$ of the task converges to a set of candidate actions $\mathcal{A}_t$ that meets the consistency criterion in \eqref{eq:threshold_rule}.
\end{corollary}
\begin{proof}
The regret bound in Prop. \ref{prop:regret_bound} implies that the LLM will eventually output a set of candidate actions that contains the optimal action with probability $1$. This implies that the set of candidate actions will only include the optimal action. By definition of the function $\lambda$ [cf.~\eqref{eq:self_consistency}], this set is consistent.
\end{proof}
\begin{remark}
The consistency threshold $\gamma$ can be adjusted to control the rate at which the ICL process converges. In practice, if the ICL process does not converge within a few iterations (e.g., 2-10), the process can be interrupted manually.
\end{remark}  
\section{Experimental Evaluation}\label{sec:evaluation}
To complement the analysis above and evaluate our framework experimentally, we apply it to the incident response use case described in \S \ref{sec:use_case}.  All artifacts related to the evaluation (i.e., data, weights, prompts, and code) are available at \cite{llm_source_kim}.

\subsection{Experimental Setup}
We evaluate our framework based on log data from $25$ incidents across $4$ different datasets published in the literature, namely CTU-Malware-2014 \cite{GARCIA2014100}, CIC-IDS-2017 \cite{icissp18}, AIT-IDS-V2-2022 \cite{ait_ids_1}, and CSLE-IDS-2024 \cite{dsn24_hammar_stadler}; see Table~\ref{tab:dataset_types}. 

Each incident is defined by logs (e.g., from an intrusion detection system) and a system description in textual form. Given this information, the task is to generate effective response actions for recovering the system from the incident. The datasets also include ground-truth response actions, which we use to evaluate the accuracy of the generated actions.

The distribution of \textsc{mitre att\&ck tactics} for the incidents in the evaluation datasets is illustrated in Fig.~\ref{fig:eval_dataset}.
\begin{figure}[H]
  \centering
  \scalebox{0.85}{
   \includegraphics{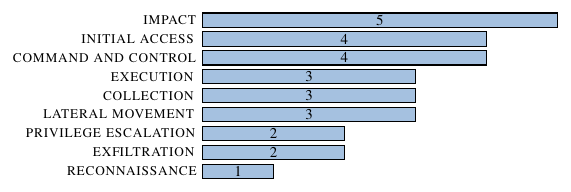}    
  }
  \caption{Number of occurrences of different \textsc{mitre att\&ck tactics} \cite{strom2018mitre} among the incidents in the evaluation datasets; cf.~Table~\ref{tab:dataset_types}.}
  \label{fig:eval_dataset}
\end{figure}
We provide a (condensed) example of an incident from the CTU-Malware-2014 dataset \cite{GARCIA2014100} to the right.

\begin{tcolorbox}[float,colback=lightgray!20, colframe=lightgray!30, sharp corners, boxrule=0pt, title=\small \textcolor{black}{\underline{Example incident from the CTU-Malware-2014 dataset \cite{GARCIA2014100}}.}]
\footnotesize  \textbf{System description (condensed)}: Two subnetworks (A and B) are connected via a switch that is also connected to the Internet. All servers run Windows XP SP2. Their IPs and configurations are...\\
  
\footnotesize \textbf{Snort alert logs (condensed)}:
\vspace{-0.2cm}
\begin{verbatim}
[120:3:2] (http_inspect) NO CONTENT-LENGTH..
[1:31033:6] MALWARE Win.Trojan.Cryptodefence..
{TCP} 147.32.84.165:1057 -> 222.88.205.195:443
[129:5:1] Bad segment, adjusted size..
[139:1:1] (spp_sdf) SDF..
\end{verbatim}
\footnotesize \textbf{Incident summary}: Server 147.32.84.165 is infected with the \textsc{win.trojan.cryptodefence} ransomware. Alerts show the server is making outbound command and control (C2) connections to 222.88.205.195. This indicates that the ransomware is active and may be preparing to encrypt files or has already begun doing so.\\

\textbf{Response actions (condensed)}:

\footnotesize 1. Disconnect the Ethernet cable of the infected server at 147.32.84.165 to sever its network connection. Concurrently, configure a rule on the main switch/firewall to block all outbound traffic to the C2 server 222.88.205.195.\\

2. Analyze the central switch to scan all network traffic from both subnetworks A and B for any other hosts attempting to make connections to the malicious IP 222.88.205.195.\\

3. Before altering the infected server, create a complete bit-for-bit forensic image of its hard drive. This preserves the ransomware executable, encrypted files, and other evidence for future analysis.\\

4. Wipe the hard drive of 147.32.84.165. If other infected machines were discovered, they must also be taken offline and wiped.\\

5. Upgrade all servers from Windows XP SP2 (which is obsolete) to a modern operating system that receives security patches.\\

6. Restore the server's data from a trusted backup. Once the server is rebuilt with a modern operating system, reconnect it to the network and closely monitor for any anomalous activity.
\normalfont
\end{tcolorbox}

\vspace{2mm}
\noindent\textit{\textbf{Evaluation metrics.}}
We consider three evaluation metrics:
\begin{enumerate}
  \item \textit{Recovery time $T$.}
\begin{itemize}
  \item This metric represents the number of actions to recover from the incident. Following the \textsc{mitre d3fend} taxonomy (see \cite{kaloroumakis2021d3fend}), we define an incident to be recovered if the following criteria are met:
\begin{itemize}
\item \textsc{contained}: the attack has been isolated and stopped from spreading.
\item \textsc{assessed}: the scope and severity of the attack have been determined.
\item \textsc{preserved}: forensic evidence related to the incident has been preserved.
\item \textsc{evicted}: the attacker's access has been revoked, and potentially malicious code or processes have been removed from the system.
\item \textsc{hardened}: the system has been hardened to prevent recurrence of the same attack.
\item \textsc{restored}: services have been restarted and user access has been restored.
\end{itemize}
For example, the optimal recovery time for the example incident in the box above is $T=6$.
\end{itemize}
\item \textit{Percentage of ineffective actions.}
\begin{itemize}
  \item An action is ineffective if it is either hallucinated [cf.~Def.~\ref{def:hallucination}] or includes unnecessary steps.
\end{itemize}
\item \textit{Percentage of failed recoveries.}
\begin{itemize}
  \item Recovery has failed when the LLM generates a sequence of actions that does not recover the system.
\end{itemize}
\end{enumerate}    
\begin{remark}
The recovery time $T$ measures the number of actions to recover, rather than the elapsed wall-clock time. Hence, $T$ quantifies the decision-making efficiency.
\end{remark}  

\vspace{2mm}
\noindent\textit{\textbf{Instantiation of our framework.}}
We instantiate our framework with $N=3$ candidate actions and use parameter $\beta=0.9$ in the consistency function; cf.~\eqref{eq:self_consistency}. To define the LLM, we use a fine-tuned version of \textsc{deepseek-r1-14b} \cite{deepseekai2025deepseekr1incentivizingreasoningcapability}, which we run on $4\times$ RTX 8000 GPUs with $4$-bit quantization. For details about the fine-tuning, see our previous work \cite{hammar2025incidentresponseplanningusing}.

To configure the consistency threshold $\gamma$ [cf.~\eqref{eq:threshold_rule}], we collect a calibration dataset of $n=100$ sets of candidate actions where the action prescribed by the policy $\pi_{\gamma}$ is hallucinated.\footnote{We obtain the samples by prompting the LLM with details about example incidents and manually checking which actions are hallucinations.} Figure \ref{fig:threshold_tuning} shows the distribution of consistency values in this dataset. Based on this distribution, we configure the consistency threshold to be $\gamma=0.9$, which ensures that the hallucination probability is upper bounded by $0.05$; cf.~Prop.~\ref{prop:abstention_bound}. For further experimental details, see Appendix~\ref{app:experimental_setup}.

\begin{figure}[H]
  \centering
  \scalebox{0.78}{
    \includegraphics{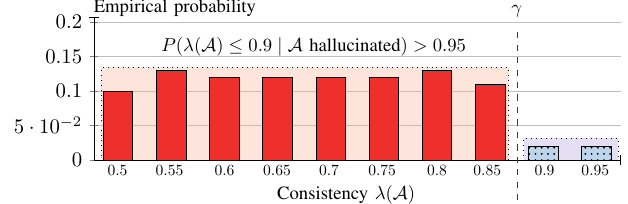}
  }
  \caption{Empirical distribution of the consistency $\lambda(\mathcal{A})$ [cf.~\eqref{eq:self_consistency}] in the calibration dataset of $n=100$ sets of candidate actions where the action prescribed by the policy $\pi_{\gamma}$ is hallucinated; cf.~\eqref{eq:threshold_rule}.}
  \label{fig:threshold_tuning}
\end{figure}

\vspace{2mm}
\noindent\textit{\textbf{Feedback.}}
If the policy $\pi_{\gamma}$ [cf.~\eqref{eq:threshold_rule}] recommends abstention, we collect feedback $R^k_t$ in terms of an evaluation of the action that would have been taken if abstention did not occur. We define this feedback to be a textual description explaining why the action is either effective or ineffective, which we then append to the context of the LLM to generate new actions.

\vspace{2mm}
\noindent\textit{\textbf{Baselines.}}
We compare our framework with three frontier LLMs: \textsc{deepseek-r1} \cite{deepseekai2025deepseekr1incentivizingreasoningcapability}, \textsc{gemini 2.5 pro} \cite{comanici2025gemini25pushingfrontier}, and \textsc{openai o3} \cite{openai2024gpt4technicalreport}. These LLMs have significantly more parameters and longer context windows than the LLM we use to instantiate our framework (\textsc{deepseek-r1-14b}); see Table \ref{tab:parameter_counts}. 

\begin{table}[H]
  \centering
  \scalebox{0.95}{
    \begin{tabular}{lll} \toprule
\rowcolor{lightgray}
      {\textit{Model}} & {\textit{Number of parameters}} & {\textit{Context window size}}  \\ \midrule
      \rowcolor{lightblue}
      \textsc{our framework} & $14$ billion & $128,000$  \\
      \textsc{deepseek-r1} \cite{deepseekai2025deepseekr1incentivizingreasoningcapability} & $671$ billion \cite {deepseekai2025deepseekr1incentivizingreasoningcapability} & $128,000$\\      
      \textsc{gemini 2.5 pro} \cite{comanici2025gemini25pushingfrontier} & unknown ($\geq 100$ billion)  & $1$ million\\
      \textsc{openai o3} \cite{openai2024gpt4technicalreport} & unknown ($\geq 100$ billion) & $200,000$\\
    \bottomrule\\
  \end{tabular}}
  \caption{Comparison between the sizes of different LLMs.}\label{tab:parameter_counts}
\end{table}

\subsection{Evaluation Results}
The results are summarized in Fig.~\ref{fig:eval_bars_1}. Across all evaluation datasets, our framework achieves the shortest recovery time. On average, the recovery time of our framework is $12.02$ compared to $16.21$ for the next best model. Among the frontier LLMs, we observe that \textsc{gemini 2.5 pro} performs best on average, whereas the difference in performance between \textsc{openai o3} and \textsc{deepseek-r1} is not statistically significant.
\begin{figure}
  \centering      
  \scalebox{0.75}{
   \includegraphics{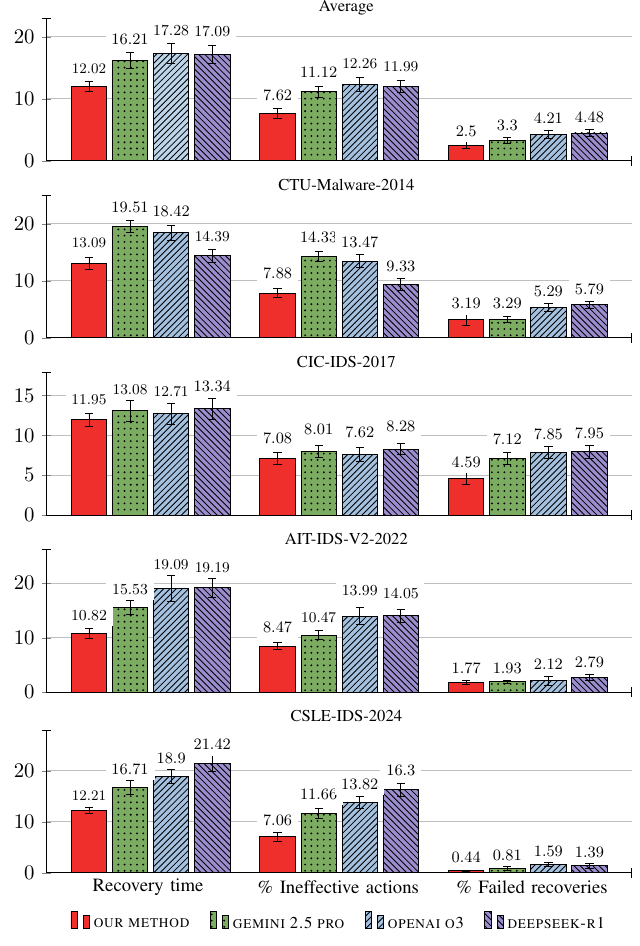}
  }
  \caption{Evaluation results ($\downarrow$ better): comparison between our framework and frontier LLMs. Bar colors relate to different models; bar groups indicate performance metrics; numbers and error bars indicate the mean and the standard deviation from $5$ evaluations with different random seeds.}
  \label{fig:eval_bars_1}
\end{figure}

To understand the relative importance of each step of our framework (i.e., lookahead, ICL, and abstention), we evaluate our framework with and without each step. The results are shown in Fig.~\ref{fig:eval_bars_2}. We observe performance degradations when each step is removed. In particular, when removing both the lookahead and the ICL, the average completion time increases from $12$ to $21$. Moreover, when removing the abstention, the average hallucination probability increases from $0.02$ to $0.06$.

Figure~\ref{fig:abstention} shows the average number of times our framework requests feedback, i.e., the number of times the consistency of the candidate actions is below the consistency threshold $\gamma=0.9$; cf.~\eqref{eq:threshold_rule}. We observe that the number of feedback requests varies significantly between the evaluation datasets, with an average of $2.24$. Moreover, we note that our framework leads to $1.13$ unnecessary feedback requests on average.

Figure~\ref{fig:regret} shows the average regret [cf.~\eqref{eq:bayes_regret}] of the ICL process in our framework. We observe that the regret plateaus after 8 iterations of ICL. This convergence indicates that a near-optimal action in each step of the response plan is identified after getting feedback on the quality of $8$ actions.

Finally, Fig.~\ref{fig:planning_scale} shows the compute time per time step of the lookahead optimization in our framework for varying numbers of candidate actions $N$. We observe that the compute time increases linearly with $N$ when the lookahead predictions are computed sequentially. However, we also find that by parallelizing the computation across multiple GPUs, the planning time remains nearly constant as $N$ increases.

\begin{figure}
  \centering
  \scalebox{0.75}{
    \includegraphics{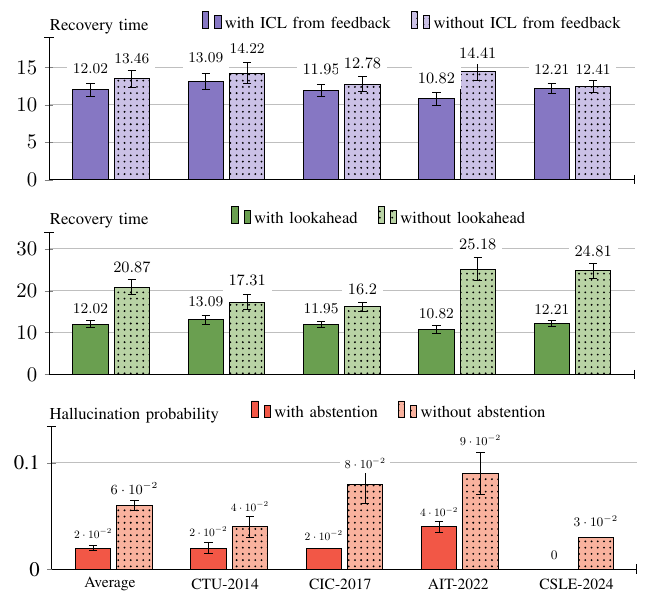}
  }
  \caption{Ablation results. Bars relate to our framework with and without ICL (upper row), lookahead (middle row), and abstention (lower row); bar groups indicate the evaluation dataset; numbers and error bars indicate the mean and the standard deviation from $5$ evaluations with different random seeds.}
  \label{fig:eval_bars_2}
\end{figure}

\begin{figure}
  \centering
  \scalebox{0.75}{
    \includegraphics{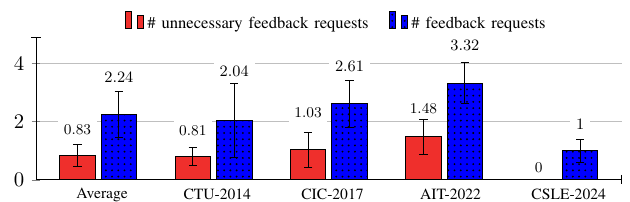}
  }
  \caption{The average number of feedback requests; bar groups indicate the evaluation dataset; numbers and error bars indicate the mean and the standard deviation from $5$ evaluations with different random seeds.}
  \label{fig:abstention}
\end{figure}

\subsection{Discussion of the Evaluation Results}
Our results show that large language models (LLMs) are promising tools for supporting security management tasks in networked systems, particularly open-ended tasks like incident response. Compared to frontier LLMs, our framework achieves consistently better performance across all evaluation metrics. This performance advantage is driven by the lookahead and ICL steps of our framework, as shown in the ablation study.

In the context of incident response, our framework can be compared against incident response playbooks that map incident scenarios to sequences of response actions \cite{10.1504/IJICS.2007.012248,playbook_response}, such as those provided by Splunk \cite{splunk_playbook}, CISA \cite{cisa_playbook}, and OASIS \cite{oasis_playbook}. Although incident response playbooks can be effective, they rely on security experts for configuration. As a consequence, they are difficult to keep up-to-date with evolving security threats and system architectures \cite{10646756}. Another common critique of playbooks is that they consist of generic response actions that are difficult for non-experts to interpret and execute effectively \cite{10.1145/3491102.3517559}. By contrast, our framework does not rely on domain experts for configuration and generates more precise and context-specific response actions. We provide a detailed comparison with playbooks in Appendix~\ref{app:playbooks}.

From a reliability perspective, the chief challenge when using LLMs as decision support is the risk of hallucination. We prove both experimentally and theoretically that our framework mitigates this risk. In particular, we have shown that the hallucination risk of our framework can be controlled by tuning the consistency threshold; see Prop.~\ref{prop:abstention_bound} and Fig.~\ref {fig:eval_bars_2}.

\begin{figure}
  \centering   
  \scalebox{0.75}{
    \includegraphics{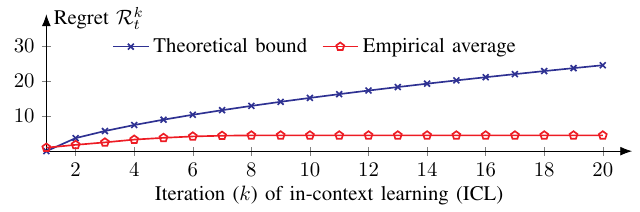}
  }
  \caption{The average regret based on $20$ iterations of ICL. The parameters $C$ and $|\mathcal{A}|$ in Prop.~\ref{prop:regret_bound} are set to $1$ and $10$, respectively.}
  \label{fig:regret}
\end{figure}

\begin{figure}
  \centering
  \scalebox{0.75}{
    \includegraphics{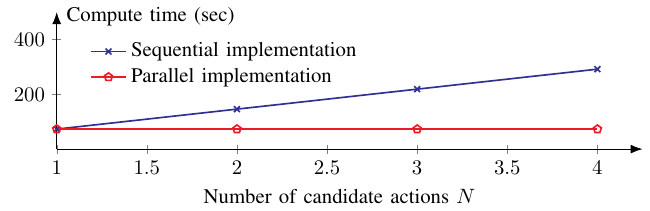}
  }
  \caption{Time required (per time step) to execute the lookahead optimization for a varying number of candidate actions $N$. The average planning times were computed based on $5$ executions with RTX 8000 GPUs.}
  \label{fig:planning_scale}
\end{figure}
\section{Conclusion}
We introduce a principled framework for using a large language model (LLM) as decision support in security management. In our framework, we use the LLM to generate candidate actions and select among them through lookahead optimization. We then evaluate the consistency of the candidate actions. If the consistency is low, our framework abstains from the selected action and instead collects external feedback, e.g., from a digital twin or a security expert. Our analysis provides theoretical guarantees on hallucination risk and regret, while our empirical evaluation demonstrates improved planning performance compared to frontier LLMs.

\vspace{2mm}
\noindent\textit{\textbf{Future work.}} The main direction of future work is to evaluate our framework on additional security tasks to better understand its empirical performance. We also plan to investigate foundation models beyond LLMs, such as time-series models.
\clearpage
\section*{Acknowledgment}
This research is supported by the Swedish Research Council under contract 2024-06436.

\appendices

\section{Example Prompt Template}\label{app:prompt}
All prompt templates that we use for the experimental evaluation are available in the repository at \cite{llm_source_kim}. We provide an example prompt template below.

\begin{tcolorbox}[float,colback=lightgray!60, colframe=lightgray!30, sharp corners, boxrule=0pt, title=\small \textcolor{black}{\underline{Prompt template for generating a response action}.}]
\footnotesize 
Below is a system description, a sequence of network logs (e.g., from an intrusion detection system), a description of a cybersecurity incident, the current state of the recovery from the incident, a list of previously executed recovery actions, and an instruction that describes a task. Write a response that appropriately completes the request. Before generating the response, think carefully about the system, the logs, and the instruction, then create a step-by-step chain of thoughts to ensure a logical and accurate response.\\\\
\#\#\# System: ...\\
$\quad$\\
\#\#\# Logs: ...\\
$\quad$\\
\#\#\# Incident: ...\\

\#\#\# State: ...\\

\#\#\# Previous recovery actions: ...\\

\#\#\# Instruction:
You are a security operator with advanced knowledge in cybersecurity and IT systems. You have been given information about a security incident and should generate the next suitable action for recovering the system from the incident. Your suggested action should be based on the logs, the system description, the current state, and the previous recovery actions only.
Make sure that the suggested recovery action is consistent with the system description and the logs and that you do not repeat any action that has already been performed.
The goal when selecting the recovery action is to change the state so that one of the state properties that is currently 'false' becomes 'true'. The ideal recovery action sequence is: 1. contain the attack 2. gather information 3. preserve evidence 4. eradicate the attacker 5. harden the system 6. recover operational services.
When selecting the recovery action, make sure that it is concrete and actionable and minimizes unnecessary service disruptions. Vague or unnecessary actions will not change the state and should be avoided.
Return a JSON object with two properties: 'Action' and 'Explanation', both of which should be strings.
The property 'Action' should be a string that concisely describes the concrete recovery action.
The property 'Explanation' should be a string that concisely explains why you selected the recovery action and motivates why the action is needed.\\
$\quad$\\
\#\#\# Response: \textless{}think\textgreater{}
\normalsize
\normalfont
\end{tcolorbox}

\section{Pseudocode}\label{app:pseudocode}
The pseudocode of our framework is listed in Alg. \ref{alg:our_method}.
\begin{algorithm}
\footnotesize
\DontPrintSemicolon
\caption{Our framework for planning with an LLM.}\label{alg:our_method}
Initialize time step $t\leftarrow 0$.\;
\While{$\text{the LLM does not think the task is completed}$}{
Initialize ICL iteration $k\leftarrow 0$.\;
Generate a set of candidate actions $\mathcal{A}^k_t$.\;
\While{$\pi_{\gamma}(\mathcal{A}^k_t) = \emptyset$}{
Collect feedback $R^k_t$ and update the LLM's context.\;
Update ICL iteration $k\leftarrow k+1$.\;
Generate a new set of candidate actions $\mathcal{A}^k_t$.\;
}
Select action $\mathbf{a}_t=\pi_{\gamma}(\mathcal{A}^k_t)$ and add it to the LLM's context.\;
Update the time step as $t \leftarrow t+1$.\;
}
Return $\mathbf{a}_0,\mathbf{a}_1,\hdots,\mathbf{a}_{T-1}$ as decision support to a security operator.
\normalsize
 \end{algorithm}

\section{Experimental Setup}\label{app:experimental_setup}
We preprocessed the evaluation datasets to ensure a consistent data format for evaluating the LLMs. Specifically, the datasets CIC-IDS-2017 \cite{icissp18} and CTU-Malware-2014 \cite{GARCIA2014100} are distributed as PCAP files. To convert these files into the textual logs required to prompt the LLM for the incident response task, we replayed the PCAP files through the Snort Intrusion Detection System (IDS) \cite{snort}. We configured Snort with the community ruleset v2.9.17.1. For the datasets that already include security logs (e.g., AIT-IDS-V2-2022 \cite{ait_ids_1} and CSLE-IDS-2024 \cite{dsn24_hammar_stadler}), we used the logs in their original format

The logs in the original evaluation datasets are labeled as either benign or malign, but they do not contain response plans. To evaluate the quality of the actions generated by the LLM, we manually constructed ground-truth response plans based on the metadata provided with each dataset. This metadata includes attack timelines, attacker IP addresses, victim IP addresses, and the specific attack vectors used. Using this information, we labeled each attack according to the MITRE ATT\&CK taxonomy and derived ground-truth response plans following standard incident response procedures. We have made all response plans for the evaluation available at \cite{llm_source_kim}.

\vspace{2mm}

\noindent\textit{\textbf{Choice of LLM.}} Our framework can be instantiated with any LLM; we choose one with 14 billion parameters due to hardware limitations. There are a few open-source models of this size, but not that many. We choose \textsc{deepseek-r1-14b} for its reasoning abilities and chain-of-thought training \cite{10.5555/3600270.3602070}. Another benefit of \textsc{deepseek-r1-14b} is that it provides a relatively long context window for a model of its size, which is advantageous for in-context learning (ICL).

\vspace{2mm}

\noindent\textit{\textbf{Justification for using the recovery time metric.}} The recovery time $T$ measures the number of response actions required to recover from the attack, rather than the elapsed wall-clock time. As such, $T$ quantifies the efficiency of the decision-making process, abstracting away from implementation-specific timing effects. This provides a general and implementation-agnostic metric for comparing different incident response plans. While no single metric can capture all dimensions of incident response, recovery time is widely used in both research and practice as an indicator of recovery performance; see e.g., \cite{hussey2025instana}. Moreover, it can be complemented with additional performance metrics that reflect specific organizational goals, such as risk preferences or service availability constraints.

The optimal recovery time is attack-dependent. For instance, in the event of a denial of service (DoS) attack, containment and eviction can often be achieved simultaneously by appropriate traffic filtering, which both isolates the attack and revokes attacker access. In contrast, an advanced persistent threat (APT) typically requires multiple actions to complete these stages, leading to longer recovery times \cite{10955193,tifs_25_HLALB}.

\section{Comparison  with Incident Response Playbooks}\label{app:playbooks}
In the context of incident response, our framework can be seen as a more dynamic complement to incident response playbooks \cite{playbook_response}, which makes it easy to adopt in existing workflows of security operation centers \cite{10.5555/3767870.3767878}. In the following, we describe the main similarities and advantages of our framework in comparison with incident response playbooks.

\vspace{2mm}

\noindent\textit{\textbf{Similarities.}}
Both the response plans generated by our framework and playbooks are expressed in text and organized around the standard stages of incident response, such as containment and eviction. Moreover, both our framework and playbooks are intended to provide decision support rather than automating the response. Another similarity is that both of them categorize attacks according to the MITRE ATT\&CK taxonomy. Despite these similarities, our framework provides several advantages over playbooks, as listed below.

\vspace{2mm}

\noindent\textit{\textbf{Advantages.}}
The main advantage is that our framework does not rely on domain experts for configuration, as playbooks do. This makes our framework more flexible. Another benefit is that our framework generates more precise and context-specific response actions than playbooks, which often prescribe vague actions that are not directly executable, as reported in several empirical studies; see e.g., \cite{10.1145/3491102.3517559, 10646756}. By contrast, our framework produces executable actions tailored to the system logs, which helps the operator to prioritize log entries. As a consequence, our framework provides a higher degree of automation than conventional playbooks, shifting the operator’s role toward validating the generated response plan rather than sifting through logs. Beyond these advantages, a fundamental difference between our framework and playbooks is that a response plan generated by our framework typically spans 2–3 pages, whereas playbooks often span $40$+ pages (see e.g., \cite{cisa_playbook}), many of which contain information that is not directly relevant to the response (e.g., general security principles). Thus, our framework serves as a more actionable and dynamic complement to conventional response playbooks.

\section{Consistency Examples}\label{app:consistency}
To illustrate the consistency function $\lambda$ [cf.~\eqref{eq:self_consistency}], consider a scenario where a security operator is responding to a potential SQL injection attack. Assume that the LLM generates the following three candidate actions (i.e., $N=3$):
\begin{itemize}
\item \textbf{$\mathbf{a}^1$}: Isolate the database server.
\item \textbf{$\mathbf{a}^2$}: Block the offending IP address.
\item \textbf{$\mathbf{a}^3$}: Update the firewall rules.
\end{itemize}

\vspace{2mm}

\noindent\textit{\textbf{Scenario 1.}} Suppose that the LLM predicts that all three actions are effective, leading to similar estimates to the time remaining to resolve the incident, e.g., $T_{t+1}^1 = 10, T_{t+1}^2 = 12$, and $T_{t+1}^3 = 11$. In this case, the average value is $11$ and the consistency with $\beta=0.9$ is 
\begin{align*}
  &\lambda(\mathcal{A}_t) = \\
  &\exp\left(\frac{-0.9\left((10-11)^2 + (12-11)^2 + (11-11)^2\right)}{3}\right)\\
                         &= 0.548.
\end{align*}

\vspace{2mm}

\noindent\textit{\textbf{Scenario 2.}} Suppose that the LLM is confused and provides wildly different estimates for the same actions, e.g., $T_{t+1}^1 = 5, T_{t+1}^2 = 60$, and $T_{t+1}^3 = 120$. In this case, the average value is $61.66$ and the consistency with $\beta=0.9$ is 
\begin{align*}
  &\lambda(\mathcal{A}_t)= \exp\bigg(\\
&  \frac{-0.9\left((5-61.66)^2 + (60-61.66)^2 + (120-61.66)^2\right)}{3}\bigg)\\
                         &\approx 0.
\end{align*}

\bibliographystyle{IEEEtran}
\bibliography{references}

@ARTICLE{hammar_stadler_tnsm,
  author={Hammar, Kim and Stadler, Rolf},
  journal={IEEE Transactions on Network and Service Management},
  title={Intrusion Prevention Through Optimal Stopping},
  year={2022},
  volume={19},
  number={3},
  pages={2333-2348},
  doi={10.1109/TNSM.2022.3176781}}

@inproceedings{snort,
author = {Roesch, Martin},
title = {Snort - Lightweight Intrusion Detection for Networks},
year = {1999},
publisher = {USENIX Association},
address = {USA},
booktitle = {Proceedings of the 13th USENIX Conference on System Administration},
pages = {229–238},
numpages = {10},
location = {Seattle, Washington},
series = {LISA '99}
}

@incollection{strom2018mitre,
  title={{Mitre ATT\&CK}: Design and philosophy},
  author={Strom, Blake E and Applebaum, Andy and Miller, Doug P and Nickels, Kathryn C and Pennington, Adam G and Thomas, Cody B},
  booktitle={Technical report},
  year={2018},
  publisher={MITRE}
}

@ARTICLE{10955193,
  author={Hammar, Kim and Li, Tao and Stadler, Rolf and Zhu, Quanyan},
  journal={IEEE Transactions on Information Forensics and Security},
  title={Adaptive Security Response Strategies Through Conjectural Online Learning},
  year={2025},
  volume={20},
  number={},
  pages={4055-4070},
  doi={10.1109/TIFS.2025.3558600}}

@Misc{wazuh,
  author = 	 {{Wazuh Inc}},
  title = 	 {Wazuh - The Open Source Security Platform},
  year = 	 2022,
  url={https://wazuh.com/}
}

@INPROCEEDINGS{dsn24_hammar_stadler,
  author={Hammar, Kim and Stadler, Rolf},
  booktitle={2024 54th Annual IEEE/IFIP International Conference on Dependable Systems and Networks (DSN)},
  title={Intrusion Tolerance for Networked Systems through Two-Level Feedback Control},
  year={2024},
  volume={},
  number={},
  pages={338-352},
  doi={10.1109/DSN58291.2024.00042}}

@misc{openai2024gpt4technicalreport,
      title={{GPT-4} Technical Report}, 
      author={OpenAI and Josh Achiam and Steven Adler and others},
      year={2024},
      eprint={2303.08774},
      archivePrefix={arXiv},
      primaryClass={cs.CL},
      url={https://arxiv.org/abs/2303.08774},
      note={\url{https://arxiv.org/abs/2303.08774}}
}

@misc{deepseekai2025deepseekr1incentivizingreasoningcapability,
      title={{DeepSeek-R1}: Incentivizing Reasoning Capability in {LLM}s via Reinforcement Learning}, 
      author={DeepSeek-AI and Daya Guo and Dejian Yang and others},
      year={2025},
      eprint={2501.12948},
      archivePrefix={arXiv},
      primaryClass={cs.CL},
      url={https://arxiv.org/abs/2501.12948}
}

@inproceedings{net_llm,
author = {Wu, Duo and Wang, Xianda and Qiao, Yaqi and Wang, Zhi and Jiang, Junchen and Cui, Shuguang and Wang, Fangxin},
title = {{NetLLM}: Adapting Large Language Models for Networking},
year = {2024},
isbn = {9798400706141},
publisher = {Association for Computing Machinery},
address = {New York, NY, USA},
url = {https://doi.org/10.1145/3651890.3672268},
doi = {10.1145/3651890.3672268},
booktitle = {Proceedings of the ACM SIGCOMM 2024 Conference},
pages = {661–678},
numpages = {18},
location = {Sydney, NSW, Australia},
series = {ACM SIGCOMM '24}
}

@inproceedings {pentest_gpt,
author = {Gelei Deng and Yi Liu and V{\'\i}ctor Mayoral-Vilches and Peng Liu and Yuekang Li and Yuan Xu and Tianwei Zhang and Yang Liu and Martin Pinzger and Stefan Rass},
title = {{PentestGPT}: Evaluating and Harnessing Large Language Models for Automated Penetration Testing},
booktitle = {33rd USENIX Security Symposium (USENIX Security 24)},
year = {2024},
isbn = {978-1-939133-44-1},
address = {Philadelphia, PA},
pages = {847--864},
url = {https://www.usenix.org/conference/usenixsecurity24/presentation/deng},
publisher = {USENIX Association},
month = aug
}

@inproceedings{10.5555/3600270.3602070,
author = {Wei, Jason and Wang, Xuezhi and Schuurmans, Dale and Bosma, Maarten and Ichter, Brian and Xia, Fei and Chi, Ed H. and Le, Quoc V. and Zhou, Denny},
title = {Chain-of-thought prompting elicits reasoning in large language models},
year = {2022},
isbn = {9781713871088},
publisher = {Curran Associates Inc.},
address = {Red Hook, NY, USA},
booktitle = {Proceedings of the 36th International Conference on Neural Information Processing Systems},
articleno = {1800},
numpages = {14},
location = {New Orleans, LA, USA},
series = {NIPS '22}
}

@inproceedings{10.5555/3692070.3693331,
author = {Liu, Zhihan and Hu, Hao and Zhang, Shenao and Guo, Hongyi and Ke, Shuqi and Liu, Boyi and Wang, Zhaoran},
title = {Reason for future, act for now: a principled architecture for autonomous {LLM} agents},
year = {2024},
publisher = {JMLR.org},
booktitle = {Proceedings of the 41st International Conference on Machine Learning},
articleno = {1261},
numpages = {76},
location = {Vienna, Austria},
series = {ICML'24}
}

@misc{google_llm_recovery,
      title={From Naptime to Big Sleep: Using Large Language Models To Catch Vulnerabilities In Real-World Code},
      author={Miltos Allamanis and Martin Arjovsky and Charles Blundell and others},
      year={2024},
      note={\url{https://googleprojectzero.blogspot.com/2024/10/from-naptime-to-big-sleep.html}}
}

@misc{rodriguez2025frameworkevaluatingemergingcyberattack,
      title={A Framework for Evaluating Emerging Cyberattack Capabilities of {AI}}, 
      author={Mikel Rodriguez and Raluca Ada Popa and Four Flynn and Lihao Liang and Allan Dafoe and Anna Wang},
      year={2025},
      eprint={2503.11917},
      archivePrefix={arXiv},
      primaryClass={cs.CR},
      url={https://arxiv.org/abs/2503.11917},
      note={\url{https://arxiv.org/abs/2503.11917}}
}

@conference{icissp18,
author={Iman Sharafaldin and Arash {Habibi Lashkari} and Ali A. Ghorbani},
title={Toward Generating a New Intrusion Detection Dataset and Intrusion Traffic Characterization},
booktitle={Proceedings of the 4th International Conference on Information Systems Security and Privacy - ICISSP},
year={2018},
pages={108-116},
publisher={SciTePress},
organization={INSTICC},
doi={10.5220/0006639801080116},
isbn={978-989-758-282-0},
issn={2184-4356},
}

@article{GARCIA2014100,
title = {An empirical comparison of botnet detection methods},
journal = {Computers \& Security},
volume = {45},
pages = {100-123},
year = {2014},
issn = {0167-4048},
doi = {https://doi.org/10.1016/j.cose.2014.05.011},
url = {https://www.sciencedirect.com/science/article/pii/S0167404814000923},
author = {S. García and M. Grill and J. Stiborek and A. Zunino}
}

@inproceedings{ait_ids_1,
author = {Landauer, Max and Skopik, Florian and Wurzenberger, Markus},
title = {Introducing a New Alert Data Set for Multi-Step Attack Analysis},
year = {2024},
isbn = {9798400709579},
publisher = {Association for Computing Machinery},
address = {New York, NY, USA},
url = {https://doi.org/10.1145/3675741.3675748},
doi = {10.1145/3675741.3675748},
booktitle = {Proceedings of the 17th Cyber Security Experimentation and Test Workshop},
pages = {41–53},
numpages = {13},
location = {Philadelphia, PA, USA},
series = {CSET '24}
}

@INPROCEEDINGS{8418627,
  author={Huang, Danny Yuxing and Aliapoulios, Maxwell Matthaios and Li, Vector Guo and Invernizzi, Luca and Bursztein, Elie and McRoberts, Kylie and Levin, Jonathan and Levchenko, Kirill and Snoeren, Alex C. and McCoy, Damon},
  booktitle={2018 IEEE Symposium on Security and Privacy (SP)},
  title={Tracking Ransomware End-to-end},
  year={2018},
  volume={},
  number={},
  pages={618-631},
  doi={10.1109/SP.2018.00047}}

@misc{wannacry_nhgs,
    author = {Amyas Morse},
    title = {Investigation: {WannaCry} cyber attack and the {NHS}},
    year = {2017},
    publisher = {National Audit Office UK},
    note= {National Audit Office UK},
}

@techreport{kaloroumakis2021d3fend,
  title        = {Toward a Knowledge Graph of Cybersecurity Countermeasures},
  author       = {Kaloroumakis, Peter E. and Smith, Michael J.},
  institution  = {The MITRE Corporation},
  address      = {Annapolis Junction, MD},
  year         = {2021},
  type         = {Technical Report},
  note         = {Approved for Public Release; Distribution Unlimited},
  url          = {https://d3fend.mitre.org/resources/D3FEND.pdf}
}

@Misc{llm_source_kim,
  author = 	 {Kim Hammar and Tansu Alpcan and Emil C. Lupu},
  title = 	 {{Supplementary material of the paper "Hallucination-Resistant Security Planning with a Large Language Model"}},
  year = 	 2026,
  note = 	 {Code for fine-tuning: \url{https://github.com/Kim-Hammar/llm_recovery}, code for our testbed: \url{https://github.com/Kim-Hammar/csle}, dataset: \url{https://huggingface.co/datasets/kimhammar/CSLE-IncidentResponse-V1}, video demonstration: \url{https://www.youtube.com/watch?v=XXo4Y6LCWk4}, fine-tuned {LLM} and prompts: \url{https://huggingface.co/kimhammar/LLMIncidentResponse}},
}

@INPROCEEDINGS{playbook_response,
  author={Applebaum, Andy and Johnson, Shawn and Limiero, Michael and Smith, Michael},
  booktitle={2018 National Cyber Summit (NCS)},
  title={Playbook Oriented Cyber Response},
  year={2018},
  volume={},
  number={},
  pages={8-15},
  doi={10.1109/NCS.2018.00007}}

@article{10.1504/IJICS.2007.012248,
author = {Stakhanova, Natalia and Basu, Samik and Wong, Johnny},
title = {A taxonomy of intrusion response systems},
year = {2007},
issue_date = {January 2007},
publisher = {Inderscience Publishers},
address = {Geneva 15, CHE},
volume = {1},
number = {1/2},
issn = {1744-1765},
url = {https://doi.org/10.1504/IJICS.2007.012248},
doi = {10.1504/IJICS.2007.012248},
journal = {Int. J. Inf. Comput. Secur.},
month = jan,
pages = {169–184},
numpages = {16}
}

@misc{tifs_25_HLALB,
      title={Adaptive Network Security Policies via Belief Aggregation and Rollout}, 
      author={Kim Hammar and Yuchao Li and Tansu Alpcan and Emil C. Lupu and Dimitri Bertsekas},
      year={2025},
      eprint={2507.15163},
      archivePrefix={arXiv},
      primaryClass={eess.SY},
      url={https://arxiv.org/abs/2507.15163}
}

@ARTICLE{10991969,
  author={Mohammadi, Hamoun and Davis, Jonathan J. and Kiely, Mitchell},
  journal={IEEE Intelligent Systems},
  title={Leveraging Large Language Models for Autonomous Cyber Defense: Insights from {CAGE-2} Simulations}, 
  year={2025},
  volume={},
  number={},
  pages={1-8},
  doi={10.1109/MIS.2025.3568209}}

@misc{castro2025largelanguagemodelsautonomous,
      title={Large Language Models are Autonomous Cyber Defenders}, 
      author={Sebastián R. Castro and Roberto Campbell and Nancy Lau and Octavio Villalobos and Jiaqi Duan and Alvaro A. Cardenas},
      year={2025},
      eprint={2505.04843},
      archivePrefix={arXiv},
      primaryClass={cs.AI},
      url={https://arxiv.org/abs/2505.04843}
}

@misc{yadkori2024mitigatingllmhallucinationsconformal,
      title={Mitigating {LLM} Hallucinations via Conformal Abstention}, 
      author={Yasin Abbasi Yadkori and Ilja Kuzborskij and David Stutz and András György and Adam Fisch and Arnaud Doucet and Iuliya Beloshapka and Wei-Hung Weng and Yao-Yuan Yang and Csaba Szepesvári and Ali Taylan Cemgil and Nenad Tomasev},
      year={2024},
      eprint={2405.01563},
      archivePrefix={arXiv},
      primaryClass={cs.LG},
      url={https://arxiv.org/abs/2405.01563}
}

@misc{yan2024dependingshouldmentoringllm,
      title={Depending on yourself when you should: Mentoring {LLM} with {RL} agents to become the master in cybersecurity games}, 
      author={Yikuan Yan and Yaolun Zhang and Keman Huang},
      year={2024},
      eprint={2403.17674},
      archivePrefix={arXiv},
      primaryClass={cs.CR},
      url={https://arxiv.org/abs/2403.17674},
      note={\url{https://arxiv.org/abs/2403.17674}}
}

@inproceedings{DBLP:conf/iclr/0002WSLCNCZ23,
  author       = {Xuezhi Wang and
                  Jason Wei and
                  Dale Schuurmans and
                  Quoc V. Le and
                  Ed H. Chi and
                  Sharan Narang and
                  Aakanksha Chowdhery and
                  Denny Zhou},
  title        = {Self-Consistency Improves Chain of Thought Reasoning in Language Models},
  booktitle    = {The Eleventh International Conference on Learning Representations,
                  {ICLR} 2023, Kigali, Rwanda, May 1-5, 2023},
  publisher    = {OpenReview.net},
  year         = {2023},
  url          = {https://openreview.net/forum?id=1PL1NIMMrw},
  bibsource    = {dblp computer science bibliography, https://dblp.org}
}

@inproceedings{DBLP:conf/ndss/0012XW00S025,
  author       = {Ye Liu and
                  Yue Xue and
                  Daoyuan Wu and
                  Yuqiang Sun and
                  Yi Li and
                  Miaolei Shi and
                  Yang Liu},
  title        = {{PropertyGPT}: {LLM}-driven Formal Verification of Smart Contracts through
                  Retrieval-Augmented Property Generation},
  booktitle    = {32nd Annual Network and Distributed System Security Symposium, {NDSS}
                  2025, San Diego, California, USA, February 24-28, 2025},
  publisher    = {The Internet Society},
  year         = {2025},
  bibsource    = {dblp computer science bibliography, https://dblp.org}
}

@inproceedings{DBLP:conf/ndss/GohilDNSR25,
  author       = {Vasudev Gohil and
                  Matthew DeLorenzo and
                  Veera Vishwa Achuta Sai Venkat Nallam and
                  Joey See and
                  Jeyavijayan Rajendran},
  title        = {{LLMPirate}: {LLM}s for Black-box Hardware {IP} Piracy},
  booktitle    = {32nd Annual Network and Distributed System Security Symposium, {NDSS}
                  2025, San Diego, California, USA, February 24-28, 2025},
  publisher    = {The Internet Society},
  year         = {2025},
  bibsource    = {dblp computer science bibliography, https://dblp.org}
}

@inproceedings{DBLP:conf/ndss/YangL0L25,
  author       = {Yi Yang and
                  Jinghua Liu and
                  Kai Chen and
                  Miaoqian Lin},
  title        = {The Midas Touch: Triggering the Capability of {LLM}s for {RM-API} Misuse
                  Detection},
  booktitle    = {32nd Annual Network and Distributed System Security Symposium, {NDSS}
                  2025, San Diego, California, USA, February 24-28, 2025},
  publisher    = {The Internet Society},
  year         = {2025},
  bibsource    = {dblp computer science bibliography, https://dblp.org}
}

@inproceedings{DBLP:conf/ndss/DengLCBWLW025,
  author       = {Jiangyi Deng and
                  Xinfeng Li and
                  Yanjiao Chen and
                  Yijie Bai and
                  Haiqin Weng and
                  Yan Liu and
                  Tao Wei and
                  Wenyuan Xu},
  title        = {{RACONTEUR:} {A} Knowledgeable, Insightful, and Portable {LLM}-Powered
                  Shell Command Explainer},
  booktitle    = {32nd Annual Network and Distributed System Security Symposium, {NDSS}
                  2025, San Diego, California, USA, February 24-28, 2025},
  publisher    = {The Internet Society},
  year         = {2025},
  bibsource    = {dblp computer science bibliography, https://dblp.org}
}

@inproceedings{DBLP:conf/ndss/StafeevRSKP25,
  author       = {Aleksei Stafeev and
                  Tim Recktenwald and
                  Gianluca De Stefano and
                  Soheil Khodayari and
                  Giancarlo Pellegrino},
  title        = {YuraScanner: Leveraging {LLM}s for Task-driven Web App Scanning},
  booktitle    = {32nd Annual Network and Distributed System Security Symposium, {NDSS}
                  2025, San Diego, California, USA, February 24-28, 2025},
  publisher    = {The Internet Society},
  year         = {2025},
  bibsource    = {dblp computer science bibliography, https://dblp.org}
}

@online{hussey2025instana,
  author       = {Stephen Hussey},
  title        = {Resolve incidents faster with {IBM} {Instana} Intelligent Incident Investigation powered by agentic {AI}},
  year         = {2025},
  month        = {June},
  day          = {29},
  publisher    = {IBM}
}

@misc{comanici2025gemini25pushingfrontier,
      title={Gemini 2.5: Pushing the Frontier with Advanced Reasoning, Multimodality, Long Context, and Next Generation Agentic Capabilities}, 
      author={Gheorghe Comanici and Eric Bieber and Mike Schaekermann and others},
      year={2025},
      eprint={2507.06261},
      archivePrefix={arXiv},
      primaryClass={cs.CL},
      url={https://arxiv.org/abs/2507.06261}
}

@inproceedings{10.1145/3491102.3517559,
author = {Stevens, Rock and Votipka, Daniel and Dykstra, Josiah and Tomlinson, Fernando and Quartararo, Erin and Ahern, Colin and Mazurek, Michelle L.},
title = {How Ready is Your Ready? Assessing the Usability of Incident Response Playbook Frameworks},
year = {2022},
isbn = {9781450391573},
publisher = {Association for Computing Machinery},
address = {New York, NY, USA},
url = {https://doi.org/10.1145/3491102.3517559},
doi = {10.1145/3491102.3517559},
booktitle = {Proceedings of the 2022 CHI Conference on Human Factors in Computing Systems},
articleno = {589},
numpages = {18},
location = {New Orleans, LA, USA},
series = {CHI '22}
}

@INPROCEEDINGS{10646756,
  author={Schlette, Daniel and Empl, Philip and Caselli, Marco and Schreck, Thomas and Pernul, Günther},
  booktitle={2024 IEEE Symposium on Security and Privacy (SP)},
  title={Do You Play It by the Books? A Study on Incident Response Playbooks and Influencing Factors}, 
  year={2024},
  volume={},
  number={},
  pages={3625-3643},
  doi={10.1109/SP54263.2024.00060}}

@misc{splunk_playbook,
    author = {Splunk},
    title = {Automate incident response with playbooks and actions in {Splunk} Mission Control},
    year = {2025}
}

@misc{cisa_playbook,
    author = {CISA},
    title = {Cybersecurity Incident \& Vulnerability Response Playbooks},
    year = {2021}
}

@misc{oasis_playbook,
    author = {OASIS},
    title = {CACAO Security Playbooks Version 2.0},
    year = {2023}
}

@inproceedings{10.1145/3597503.3639121,
author = {Xia, Chunqiu Steven and Paltenghi, Matteo and Le Tian, Jia and Pradel, Michael and Zhang, Lingming},
title = {Fuzz4All: Universal Fuzzing with Large Language Models},
year = {2024},
isbn = {9798400702174},
publisher = {Association for Computing Machinery},
address = {New York, NY, USA},
url = {https://doi.org/10.1145/3597503.3639121},
doi = {10.1145/3597503.3639121},
booktitle = {Proceedings of the IEEE/ACM 46th International Conference on Software Engineering},
articleno = {126},
numpages = {13},
location = {Lisbon, Portugal},
series = {ICSE '24}
}

@inproceedings{DBLP:conf/ndss/LiuY0L25,
  author       = {Jinghua Liu and
                  Yi Yang and
                  Kai Chen and
                  Miaoqian Lin},
  title        = {Generating {API} Parameter Security Rules with {LLM} for {API} Misuse
                  Detection},
  booktitle    = {32nd Annual Network and Distributed System Security Symposium, {NDSS}
                  2025, San Diego, California, USA, February 24-28, 2025},
  publisher    = {The Internet Society},
  year         = {2025},
  bibsource    = {dblp computer science bibliography, https://dblp.org}
}

@inproceedings{DBLP:conf/ndss/HuL024,
  author       = {Peiwei Hu and
                  Ruigang Liang and
                  Kai Chen},
  title        = {{DeGPT:} Optimizing Decompiler Output with {LLM}},
  booktitle    = {31st Annual Network and Distributed System Security Symposium, {NDSS}
                  2024, San Diego, California, USA, February 26 - March 1, 2024},
  publisher    = {The Internet Society},
  year         = {2024},
  bibsource    = {dblp computer science bibliography, https://dblp.org}
}

@misc{chen2023universalselfconsistencylargelanguage,
      title={Universal Self-Consistency for Large Language Model Generation}, 
      author={Xinyun Chen and Renat Aksitov and Uri Alon and Jie Ren and Kefan Xiao and Pengcheng Yin and Sushant Prakash and Charles Sutton and Xuezhi Wang and Denny Zhou},
      year={2023},
      eprint={2311.17311},
      archivePrefix={arXiv},
      primaryClass={cs.CL},
      url={https://arxiv.org/abs/2311.17311}
}

@inproceedings{weng-etal-2023-large,
    title = "Large Language Models are Better Reasoners with Self-Verification",
    author = "Weng, Yixuan  and
      Zhu, Minjun  and
      Xia, Fei  and
      Li, Bin  and
      He, Shizhu  and
      Liu, Shengping  and
      Sun, Bin  and
      Liu, Kang  and
      Zhao, Jun",
    editor = "Bouamor, Houda  and
      Pino, Juan  and
      Bali, Kalika",
    booktitle = "Findings of the Association for Computational Linguistics: EMNLP 2023",
    month = dec,
    year = "2023",
    address = "Singapore",
    publisher = "Association for Computational Linguistics",
    url = "https://aclanthology.org/2023.findings-emnlp.167/",
    doi = "10.18653/v1/2023.findings-emnlp.167",
    pages = "2550--2575"
}

@article{ARAZZI2025100765,
title = {{NLP}-based techniques for Cyber Threat Intelligence},
journal = {Computer Science Review},
volume = {58},
pages = {100765},
year = {2025},
issn = {1574-0137},
doi = {https://doi.org/10.1016/j.cosrev.2025.100765},
url = {https://www.sciencedirect.com/science/article/pii/S1574013725000413},
author = {Marco Arazzi and Dincy {R. Arikkat} and Serena Nicolazzo and Antonino Nocera and Rafidha {Rehiman K.A.} and Vinod P. and Mauro Conti}
}

@inproceedings {299896,
author = {Xiaoyue Ma and Lannan Luo and Qiang Zeng},
title = {From One Thousand Pages of Specification to Unveiling Hidden Bugs: Large Language Model Assisted Fuzzing of Matter {IoT} Devices},
booktitle = {33rd USENIX Security Symposium (USENIX Security 24)},
year = {2024},
isbn = {978-1-939133-44-1},
address = {Philadelphia, PA},
pages = {4783--4800},
url = {https://www.usenix.org/conference/usenixsecurity24/presentation/ma-xiaoyue},
publisher = {USENIX Association},
month = aug
}

@inproceedings {299549,
author = {Peiyu Liu and Junming Liu and Lirong Fu and Kangjie Lu and Yifan Xia and Xuhong Zhang and Wenzhi Chen and Haiqin Weng and Shouling Ji and Wenhai Wang},
title = {Exploring {ChatGPT{\textquoteright}s} Capabilities on Vulnerability Management},
booktitle = {33rd USENIX Security Symposium (USENIX Security 24)},
year = {2024},
isbn = {978-1-939133-44-1},
address = {Philadelphia, PA},
pages = {811--828},
url = {https://www.usenix.org/conference/usenixsecurity24/presentation/liu-peiyu},
publisher = {USENIX Association},
month = aug
}

@inproceedings{hammar2025incidentresponseplanningusing,
  author = 	 {Kim Hammar and Tansu Alpcan and Emil C. Lupu},
      title={Incident Response Planning Using a Lightweight Large Language Model with Reduced Hallucination}, 
  booktitle    = {33rd Annual Network and Distributed System Security Symposium, {NDSS}
                  2026, San Diego, California, USA, February 23-27, 2026},
  publisher    = {The Internet Society},		  
  year = 	 2026
}

@inproceedings{10.5555/3495724.3495883,
author = {Brown, Tom B. and Mann, Benjamin and Ryder, Nick and Subbiah, Melanie and Kaplan, Jared and Dhariwal, Prafulla and Neelakantan, Arvind and Shyam, Pranav and Sastry, Girish and Askell, Amanda and Agarwal, Sandhini and Herbert-Voss, Ariel and Krueger, Gretchen and Henighan, Tom and Child, Rewon and Ramesh, Aditya and Ziegler, Daniel M. and Wu, Jeffrey and Winter, Clemens and Hesse, Christopher and Chen, Mark and Sigler, Eric and Litwin, Mateusz and Gray, Scott and Chess, Benjamin and Clark, Jack and Berner, Christopher and McCandlish, Sam and Radford, Alec and Sutskever, Ilya and Amodei, Dario},
title = {Language models are few-shot learners},
year = {2020},
isbn = {9781713829546},
publisher = {Curran Associates Inc.},
address = {Red Hook, NY, USA},
booktitle = {Proceedings of the 34th International Conference on Neural Information Processing Systems},
articleno = {159},
numpages = {25},
location = {Vancouver, BC, Canada},
series = {NIPS '20}
}

@inproceedings{dong-etal-2024-survey,
    title = "A Survey on In-context Learning",
    author = "Dong, Qingxiu  and
      Li, Lei  and
      Dai, Damai  and
      Zheng, Ce  and
      Ma, Jingyuan  and
      Li, Rui  and
      Xia, Heming  and
      Xu, Jingjing  and
      Wu, Zhiyong  and
      Chang, Baobao  and
      Sun, Xu  and
      Li, Lei  and
      Sui, Zhifang",
    editor = "Al-Onaizan, Yaser  and
      Bansal, Mohit  and
      Chen, Yun-Nung",
    booktitle = "Proceedings of the 2024 Conference on Empirical Methods in Natural Language Processing",
    month = nov,
    year = "2024",
    address = "Miami, Florida, USA",
    publisher = "Association for Computational Linguistics",
    url = "https://aclanthology.org/2024.emnlp-main.64/",
    doi = "10.18653/v1/2024.emnlp-main.64",
    pages = "1107--1128"
}

@inproceedings{10.5555/3692070.3692580,
author = {Falck, Fabian and Wang, Ziyu and Holmes, Chris},
title = {Is in-context learning in large language models {Bayesian}? a martingale perspective},
year = {2024},
publisher = {JMLR.org},
booktitle = {Proceedings of the 41st International Conference on Machine Learning},
articleno = {510},
numpages = {22},
location = {Vienna, Austria},
series = {ICML'24}
}

@misc{xie2022explanationincontextlearningimplicit,
      title={An Explanation of In-context Learning as Implicit {Bayesian} Inference}, 
      author={Sang Michael Xie and Aditi Raghunathan and Percy Liang and Tengyu Ma},
      year={2022},
      eprint={2111.02080},
      archivePrefix={arXiv},
      primaryClass={cs.CL},
      url={https://arxiv.org/abs/2111.02080}, 
}

@article{thompson,
 ISSN = {00063444},
 URL = {http://www.jstor.org/stable/2332286},
 author = {William R. Thompson},
 journal = {Biometrika},
 number = {3/4},
 pages = {285--294},
 publisher = {[Oxford University Press, Biometrika Trust]},
 title = {On the Likelihood that One Unknown Probability Exceeds Another in View of the Evidence of Two Samples},
 urldate = {2025-09-11},
 volume = {25},
 year = {1933}
}

@article{bandit_book,
  author = {Lattimore, Tor and Szepesvari, Csaba},
  description = {Bandit Algorithms},
  title = {Bandit Algorithms},
  url = {https://tor-lattimore.com/downloads/book/book.pdf},
  year = 2017
}

@INPROCEEDINGS{10154288,
  author={Hammar, Kim and Stadler, Rolf},
  booktitle={NOMS 2023-2023 IEEE/IFIP Network Operations and Management Symposium},
  title={Digital Twins for Security Automation},
  year={2023},
  volume={},
  number={},
  pages={1-6},
  doi={10.1109/NOMS56928.2023.10154288}}

@article{elazar-etal-2021-measuring,
    title = "Measuring and Improving Consistency in Pretrained Language Models",
    author = {Elazar, Yanai  and
      Kassner, Nora  and
      Ravfogel, Shauli  and
      Ravichander, Abhilasha  and
      Hovy, Eduard  and
      Sch{\"u}tze, Hinrich  and
      Goldberg, Yoav},
    editor = "Roark, Brian  and
      Nenkova, Ani",
    journal = "Transactions of the Association for Computational Linguistics",
    volume = "9",
    year = "2021",
    address = "Cambridge, MA",
    publisher = "MIT Press",
    url = "https://aclanthology.org/2021.tacl-1.60/",
    doi = "10.1162/tacl_a_00410",
    pages = "1012--1031"
}

@misc{tayebati2025learningconformalabstentionpolicies,
      title={Learning Conformal Abstention Policies for Adaptive Risk Management in Large Language and Vision-Language Models}, 
      author={Sina Tayebati and Divake Kumar and Nastaran Darabi and Dinithi Jayasuriya and Ranganath Krishnan and Amit Ranjan Trivedi},
      year={2025},
      eprint={2502.06884},
      archivePrefix={arXiv},
      primaryClass={cs.LG},
      url={https://arxiv.org/abs/2502.06884}, 
}

@misc{moeini2025surveyincontextreinforcementlearning,
      title={A Survey of In-Context Reinforcement Learning}, 
      author={Amir Moeini and Jiuqi Wang and Jacob Beck and Ethan Blaser and Shimon Whiteson and Rohan Chandra and Shangtong Zhang},
      year={2025},
      eprint={2502.07978},
      archivePrefix={arXiv},
      primaryClass={cs.LG},
      url={https://arxiv.org/abs/2502.07978}, 
}

@techreport{ISC2_2024_Workforce_Study,
  author       = {{ISC2}},
  title        = {2024 {Cybersecurity} Workforce Study},
  institution  = {ISC2},
  year         = {2024},
  month        = {October},
  url          = {https://www.isc2.org/Insights/2024/10/ISC2-2024-Cybersecurity-Workforce-Study},
  note         = {Based on an online survey of 15,852 cybersecurity professionals conducted in April–May 2024 in collaboration with Forrester Research\, Inc.; accessed 2025-09-07},
}

@online{google_ai_sec,
  author       = {Google},
  title        = {Introducing {Google} Security Operations: Intel-driven, {AI}-powered {SecOps}},
  year         = {2024},
  month        = {May},
  day          = {7},
  publisher    = {Google}
}

@INPROCEEDINGS{10575237,
  author={Van Tu, Nguyen and Yoo, Jae-Hyoung and Hong, James Won-Ki},
  booktitle={NOMS 2024-2024 IEEE Network Operations and Management Symposium},
  title={Towards Intent-based Configuration for Network Function Virtualization using In-context Learning in Large Language Models},
  year={2024},
  volume={},
  number={},
  pages={1-8},
  doi={10.1109/NOMS59830.2024.10575237}}

@INPROCEEDINGS{11073595,
  author={Dzeparoska, Kristina and Leon-Garcia, Alberto},
  booktitle={NOMS 2025-2025 IEEE Network Operations and Management Symposium},
  title={{KPI} Assurance and {LLM}s for Intent-Based Management},
  year={2025},
  volume={},
  number={},
  pages={1-9},
  doi={10.1109/NOMS57970.2025.11073595}}

@INPROCEEDINGS{11073607,
  author={Nam, Sukhyun and Van Tu, Nguyen and Hong, James Won-Ki},
  booktitle={NOMS 2025-2025 IEEE Network Operations and Management Symposium},
  title={{LLM}-Based {AI} Agent for {VNF} Deployment in {OpenStack} Environment},
  year={2025},
  volume={},
  number={},
  pages={1-7},
  doi={10.1109/NOMS57970.2025.11073607}}

@INPROCEEDINGS{11073741,
  author={Nickel, Lucas Immanuel and Hohmann, Lorenz and Stolbov, Nick and Gerstacker, Lukas and Rieger, Sebastian},
  booktitle={NOMS 2025-2025 IEEE Network Operations and Management Symposium},
  title={Integrating {LLMs} with {NetBox} and {Netmiko} for Vendor-Agnostic Intent-Based Networking},
  year={2025},
  volume={},
  number={},
  pages={1-6},
  doi={10.1109/NOMS57970.2025.11073741}}

@misc{kalai2025languagemodelshallucinate,
      title={Why Language Models Hallucinate}, 
      author={Adam Tauman Kalai and Ofir Nachum and Santosh S. Vempala and Edwin Zhang},
      year={2025},
      eprint={2509.04664},
      archivePrefix={arXiv},
      primaryClass={cs.CL},
      url={https://arxiv.org/abs/2509.04664}, 
}

@inproceedings{10.5555/3767870.3767878,
author = {Kramer, Diana and Rosique, Lambert and Narotam, Ajay and Bursztein, Elie and Kelley, Patrick Gage and Thomas, Kurt and Woodruff, Allison},
title = {Integrating large language models into security incident response},
year = {2025},
isbn = {978-1-939133-51-9},
publisher = {USENIX Association},
address = {USA},
booktitle = {Proceedings of the Twenty-First USENIX Conference on Usable Privacy and Security},
articleno = {8},
numpages = {16},
location = {Seattle, WA, USA},
series = {SOUPS '25}
}

@inproceedings{10.1145/3719027.3744872,
author = {Deng, Gelei and Ou, Haoran and Liu, Yi and Zhang, Jie and Zhang, Tianwei and Liu, Yang},
title = {Oedipus: {LLM}-enchanced Reasoning {CAPTCHA} Solver},
year = {2025},
isbn = {9798400715259},
publisher = {Association for Computing Machinery},
address = {New York, NY, USA},
url = {https://doi.org/10.1145/3719027.3744872},
doi = {10.1145/3719027.3744872},
booktitle = {Proceedings of the 2025 ACM SIGSAC Conference on Computer and Communications Security},
pages = {6–20},
numpages = {15},
location = {Taipei, Taiwan},
series = {CCS '25}
}

@inproceedings{10.1145/3719027.3765219,
author = {Aly, Ahmed and Mansour, Essam and Youssef, Amr},
title = {{OCR-APT}: Reconstructing {APT} Stories from Audit Logs using Subgraph Anomaly Detection and {LLM}s},
year = {2025},
isbn = {9798400715259},
publisher = {Association for Computing Machinery},
address = {New York, NY, USA},
url = {https://doi.org/10.1145/3719027.3765219},
doi = {10.1145/3719027.3765219},
booktitle = {Proceedings of the 2025 ACM SIGSAC Conference on Computer and Communications Security},
pages = {261–275},
numpages = {15},
location = {Taipei, Taiwan},
series = {CCS '25}
}

@inproceedings{10.1145/3719027.3744855,
author = {Ji, Zimo and Wu, Daoyuan and Jiang, Wenyuan and Ma, Pingchuan and Li, Zongjie and Wang, Shuai},
title = {Measuring and Augmenting Large Language Models for Solving Capture-the-Flag Challenges},
year = {2025},
isbn = {9798400715259},
publisher = {Association for Computing Machinery},
address = {New York, NY, USA},
url = {https://doi.org/10.1145/3719027.3744855},
doi = {10.1145/3719027.3744855},
booktitle = {Proceedings of the 2025 ACM SIGSAC Conference on Computer and Communications Security},
pages = {603–617},
numpages = {15},
location = {Taipei, Taiwan},
series = {CCS '25}
}

@inproceedings {299740,
author = {Yuexin Li and Chengyu Huang and Shumin Deng and Mei Lin Lock and Tri Cao and Nay Oo and Hoon Wei Lim and Bryan Hooi},
title = {{KnowPhish}: Large Language Models Meet Multimodal Knowledge Graphs for Enhancing {Reference-Based} Phishing Detection},
booktitle = {33rd USENIX Security Symposium (USENIX Security 24)},
year = {2024},
isbn = {978-1-939133-44-1},
address = {Philadelphia, PA},
pages = {793--810},
url = {https://www.usenix.org/conference/usenixsecurity24/presentation/li-yuexin},
publisher = {USENIX Association},
month = aug
}
\end{document}